\def\year{2021}\relax
\documentclass[letterpaper]{article} % DO NOT CHANGE THIS
\usepackage{aaai21}  % DO NOT CHANGE THIS
\usepackage{times}  % DO NOT CHANGE THIS
\usepackage{helvet} % DO NOT CHANGE THIS
\usepackage{courier}  % DO NOT CHANGE THIS
\usepackage[hyphens]{url}  % DO NOT CHANGE THIS
\usepackage{graphicx} % DO NOT CHANGE THIS
\usepackage{natbib}  % DO NOT CHANGE THIS AND DO NOT ADD ANY OPTIONS TO IT
\usepackage{caption} % DO NOT CHANGE THIS AND DO NOT ADD ANY OPTIONS TO IT
\usepackage{amsmath}

\usepackage{amsmath,amssymb,amsfonts,amsthm, comment}
\usepackage{bm, natbib}

\usepackage{shortbold}

\renewcommand{\argmax}{\mathrm{arg\,max}}
\newtheorem{theorem}{Theorem}
\usepackage{algorithm,algorithmic}
\newcommand{\etal}{\textit{et al.}}
\newcommand{\ie}{\textit{i}.\textit{e}.}

\newcommand\norm[1]{\left\lVert#1\right\rVert}

\usepackage{mathtools}

\newtheorem{lemma}{Lemma}
\newtheorem{assumption}{Assumption}
\usepackage{natbib}

\usepackage[font=small,labelfont=bf]{caption}
\usepackage{multirow}

\usepackage{color}

\usepackage[switch]{lineno}

\title{Amata: An Annealing Mechanism for Adversarial Training Acceleration}
\author {
    Nanyang Ye,\textsuperscript{\rm 1}
    Qianxiao Li ,\textsuperscript{\rm 2, \rm 5}
    Xiao-Yun Zhou, \textsuperscript{\rm 3}
    Zhanxing Zhu \footnote{Corresponding author} \textsuperscript{\rm 4}\\
}
\affiliations {
    \textsuperscript{\rm 1} Shanghai Jiao Tong University \\
    \textsuperscript{\rm 2} National University of Singapore \\
    \textsuperscript{\rm 3} PAII INC \\
    \textsuperscript{\rm 4} Peking University \\
    \textsuperscript{\rm 5} Institute of High Performance Computing, A*STAR \\
    ynylincoln@sjtu.edu.cn, qianxiao@nus.edu.sg, xiaoyun.zhou27@gmail.com,
    zhanxing.zhu@pku.edu.cn.
}
\begin{document}
%\linenumbers
\maketitle

\begin{abstract}
Despite the empirical success in various domains, it has been revealed that deep neural networks are vulnerable to maliciously perturbed input data that much degrade their performance. This is known as adversarial attacks. To counter adversarial attacks, adversarial training formulated as a form of robust optimization has been demonstrated to be effective. However, conducting adversarial training brings much computational overhead compared with standard training. In order to reduce the computational cost, we propose an annealing mechanism, Amata, to reduce the overhead associated with adversarial training. The proposed Amata is provably convergent, well-motivated from the lens of optimal control theory and can be combined with existing acceleration methods to further enhance performance. It is demonstrated that on standard datasets, Amata can achieve similar or better robustness with around 1/3 to 1/2 the computational time compared with traditional methods. In addition, Amata can be incorporated into  other adversarial training acceleration algorithms (e.g. YOPO, Free, Fast, and ATTA), which leads to further reduction in computational time on large-scale problems.
\end{abstract}

\section{Introduction}
%Recently, the revival of deep neural networks has led to breakthroughs in various fields, including computer vision, natural language processing, game playing, etc. Despite these advancements,
Deep neural networks were found to be vulnerable to malicious perturbations on the original input data. While the perturbations remain almost imperceptible to humans, they can lead to wrong predictions over the perturbed examples~\cite{szegedy2013intriguing,goodfellow2014explaining,akhtar2018threat}. These maliciously crafted examples are known as adversarial examples, which have caused serious concerns over the reliability and security of deep learning systems, particularly when deployed in life-critical scenarios, such as autonomous driving systems and medical domains.

Several defense mechanisms have been proposed, such as input reconstruction \cite{Meng2017MagNetAT,song2018pixeldefend}, input encoding \cite{buckman2018thermometer}, and adversarial training~\cite{goodfellow2014explaining,tramer2017ensemble,he2017adversarial,madry2017towards}. Among these methods, adversarial training is one of the most effective defense methods so far.  It can be posed as a robust optimization problem~\cite{ben1998robust}, where a min-max optimization problem is solved~\cite{madry2017towards,kolter2017provable}. For example, given a $C$-class dataset $S=\{(\xB^{0}_{i}, y_{i})\}^{n}_{i=1}$ with $\xB^{0}_{i} \in \Rcal^{d}$ as a normal or clean example in the $d$-dimensional input space and $y_{i} \in \Rcal^{C}$ as its associated one-hot label, the objective of adversarial training is to solve the following \textit{min-max optimization} problem:
\begin{equation}
    \label{minmax_obj}
    \min_{\thetaB} \frac{1}{N} \sum_{i=1}^{N} \max_{\norm{\xB_{i}-\xB^{0}_{i}}\leq \epsilon} \ell(h_{\thetaB}(\xB_{i}), y_{i})
\end{equation}
where $h_{\thetaB}:\Rcal^{d} \rightarrow \Rcal^{C}$ is the deep neural network (DNN) function, $\ell$ is the loss function and $\epsilon$ controls the maximum perturbation magnitude. The \textit{inner maximization} problem is to find an adversarial example $\xB_{i}$, within the $\epsilon$-ball around a given normal example $\xB^{0}_i$ that maximizes the surrogate loss $\ell$ for classification error. The \textit{outer minimization} problem is to find model parameters that minimizes the loss $\ell$ on the adversarial examples $\{\xB_{i}\}^{n}_{i=1}$ that are generated from the inner maximization. \textcolor{black}{Compared with a rich body of non-convex optimization algorithms for neural networks \cite{Goodfellow2016, kingma2014method, Hengyue2015}, designing efficient algorithm to solve the min-max problem to achieve robustness is relatively less studied.}

The inner maximization problem is typically solved by projected gradient descent (PGD). PGD perturbs a normal example $\xB^{0}$ by iteratively updating it in approximately the steepest ascent direction for a total of $K$ times. Each ascent step is modulated by a small step size and a projection step back onto the $\epsilon$-ball of $\xB^{0}$ to prevent the updated value from falling outside the $\epsilon$-ball of $\xB^{0}$ \citep{madry2017towards}:
\begin{equation}
    \xB^{k} = \prod \left(\xB^{k-1} + \alpha \cdot \text{sign}(\down_{\xB}\ell(h_{\thetaB}(\xB^{k-1}), y) \right)
\end{equation}
where $\alpha$ is the step size, $\prod(\cdot)$ is the orthogonal projection function onto $\{ \xB' : \| \xB^{0} - \xB' \| \leq \epsilon \}$, and $\xB^{k}$ is the adversarial example at $k$-th step \footnote{Note that our methodology can also be applied to single step methods, such as Fast. We take PGD as the first example for clarity.}.

A major issue limiting the practical applicability of adversarial training is the huge computational burden associated with the inner maximization steps: we need to iteratively solve the inner maximization problem to find good adversarial examples for DNN to be robust. Recently, to accelerate adversarial training, a few methods have been proposed. For example, YOPO estimated the gradient on the input by only propagating the first layer \cite{zhang2019you}, parallel adversarial training utilized multiple graphics processing units (GPUs) for acceleration \cite{Bhat2019} and running PGD-1 for multiple steps to reuse gradients \cite{Ali2019}.

% A common drawback of these methods are their implementation complexity or reduced accuracy compared with the standard PGD, which could be problematic to apply them in real practice for security sensitive systems. On the other hand, an empirical observation made by \citep{wang19i}, indicated that we might not need to find "good" solutions to the inner maximization at the initial stages of adversarial training to achieve even better robustness.

Orthogonal to these approaches, in this paper we consider accelerating adversarial training by adjusting the number of inner maximization steps as training proceeds. This is in line with an empirical observation made by \cite{wang19i}, indicating that we might not need to find "good" solutions to the inner maximization at the initial stages of adversarial training to achieve better robustness. Hence, by varying the extent the inner maximization problem is solved, we may reduce the amount of wasted computation. This forms the basis of our proposed Annealing Mechanism for Adversarial Training Acceleration, named as \textbf{Amata}. Compared with traditional methods, Amata takes 1/3 to 1/2 the time to achieve comparable or slightly better robustness. Moreover, the general applicability of annealing procedures allows for effective combination of Amata with existing acceleration approaches. On the theoretical side, Amata is shown to converge. Furthermore, as adaptive training algorithms can often be interpretted as optimal control problems~\cite{li2017stochastic,li2019foundations}, we also develop a control theoretic viewpoint of general adversarial training. Under this framework, we can motivate the qualitative form of the Amata's annealing scheme based on loss landscapes. Furthermore,  a new criterion based on the Pontryagin's maximum principle can also be derived to quantify the approximate optimality of annealing schedules.
%, in addition to the extensive experiments we provide.

% To motivate the present approach, we develop a general formulation of adversarial training as an optimal control problem, from which an approximate optimality criterion can be derived based on the Pontryagin's maximum principle. This criterion is related to but different from the empirical FOSC criterion \citep{wang19i} that improves the robustness of adversarial training by determining the number of adversarial training steps according to the empirical criterion. We will also demonstrate in our experiments that unlike FOSC, the algorithm derived from our proposed criterion leads to acceleration without decreasing the performance.

In summary, our contributions are as follows:
\begin{enumerate}
    \item We propose an adversarial training algorithm, Amata, based on annealing the inner maximization steps to reduce computation. The method is shown to be effective on benchmarks, including MNIST, CIFAR10, Caltech256, and the large-scale ImageNet dataset.
    \item As Amata is largely orthogonal to existing acceleration methods for adversarial training, it can be easily combined with them to further decrease computation. The combination of Amata with YOPO \cite{zhang2019you}, with adversarial training for free \cite{Ali2019}, with fast adversarial training \cite{Wong2020Fast}, and with adversarial training with transferable adversarial examples \cite{zheng2019efficient} is demonstrated.
    \item On the theoretical side, we prove the convergence of Amata. Moreover, we develop a general optimal control framework for annealed adversarial training, from which we can use the optimal control to qualitatively and quantitatively justify the proposed annealing schedule. This framework is also potentially useful as a basic formulation for future work on adaptive adversarial training methods.
\end{enumerate}

\section{Accelerating Adversarial Training by Annealing}

% In this section, we first introduce a general formulation of annealed adversarial training as an optimal control problem, whose formal solution is an annealing schedule satisfying the classical Pontryagin's maximum principle (PMP)~\citep{boltyanskii1960theory}. We then show how this motivates a simple annealing schedule that forms the basis of our proposed algorithm, whose convergence can be proved. In addition, we establish a criterion based on the PMP that allows us to test the approximate optimality of annealing algorithms, further justifying our annealing method.

In this section, we first introduce the proposed Amata, which aims to balance the computational cost and the accuracy of solving the inner maximization problem. A proof of the convergence of the algorithm can be found in the Appendix.
Moreover, we introduce an optimal control formulation of general annealed adversarial training, from which one can elucidate the motivation and working principles of Amata, both qualitatively and quantitatively.

\subsection{Proposed annealing adversarial training algorithm}

\begin{algorithm}[!h]
  \caption{An instantatiation of Amata for PGD}
\begin{algorithmic}
\label{algo:aal}
   \STATE \textbf{Input: } $T$: training epochs; $K_{\text{min}}/K_{\text{max}}$: the minimal/maximal number of adversarial perturbations;  $\thetaB$: parameter of neural network to be adversarially trained; $\Bcal$: mini-batch; $\alpha$:  step size for  adversarial training; $\eta$: learning rate of neural network parameters. $\tau$: constant, maximum perturbation:$\epsilon$.
   \STATE \textbf{Initialization} $\thetaB=\thetaB_{0}$
   \FOR{$t = 0$ to $T-1$}
        \STATE Compute the annealing number of adversarial perturbations: $K_{t} = K_{\text{min}} + (K_{\text{max}}-K_{\text{min}})\cdot \frac{t}{T} \nonumber$

        \STATE Compute adversarial perturbation step size: $ \alpha_{t} = \frac{\tau}{K_{t}}$
        \FOR{each mini-batch $\xB^{0}_{\Bcal}$}
            \FOR{$k=1$ to $K_{t}$}
                \STATE Compute adversarial perturbations:
                \begin{align}
                    &\xB^{k}_{\Bcal} = \xB^{k-1}_{\Bcal} + \alpha_{t} \cdot \text{sign}(\down_{\xB} \ell(h_{\thetaB}(\xB^{k}_{\Bcal}), y), \nonumber \\
                    &\xB^{k}_{\Bcal} = \mathrm{clip}(\xB^{k}_{\Bcal}, \xB^{0}_{\Bcal}-\epsilon,  \xB^{0}_{\Bcal}+\epsilon) \nonumber
                \end{align}
            \ENDFOR
            \STATE $\thetaB_{t+1} =\thetaB_{t} -\eta \down_{\thetaB} \ell(h_{\thetaB_t}(\xB^{K_{t}}_{\Bcal}), y)$
        \ENDFOR
   \ENDFOR
   \STATE Collect $\thetaB_T$ as the parameter of adversarially-trained neural network.
\end{algorithmic}
\end{algorithm}

To set the stage we first summarize the proposed algorithm (Amata) in Algorithm~\ref{algo:aal}. The intuition behind Amata is that, at the initial stage, the neural network focuses on learning features, which might not require very accurate adversarial examples. Therefore, we only need a coarse approximation of the inner maximization problem solutions. With this consideration, a small number of update steps $K$ with a large step size $\alpha$ is used for inner maximization at the beginning, and then gradually increase $K$ and decrease $\alpha$ to improve the quality of inner maximization solutions. This adaptive annealing mechanism would reduce the computational cost in the early iterations while still maintaining reasonable accuracy for the entire optimization. \emph{Note that this algorithm is only an instantiation of this mechanism on PGD and the mechanism can also be seamlessly incorporated into other acceleration algorithms.} This will be demonstrated later\footnote{Amata can be also applied to other algorithms, such as YOPO, Free, and Fast, we use PGD as an example for clarity.}.

Next, we will show the sketch for proving the convergence of the algorithm with details in the Appendix. We denote $\xB^{*}_{i}(\thetaB)=\argmax_{\xB_{i} \in \Xcal^{i}} \ell(\thetaB, \xB_{i})$ where $\ell(\thetaB, \xB_{i})$ is a short hand notation for the classification loss function $\ell(h_{\thetaB}(\xB_{i}), y_{i})$, and $\Xcal^{i}$ is the permitted perturbation range for $\xB_{i}$. Before we prove the convergence of the algorithm, we have the following assumptions which are commonly used in literature for studying convergence of deep learning algorithms~\cite{Ruiqi2019,wang19i}. 

\begin{assumption}
\label{assum:continuous_paper}
The function $\ell(\thetaB, \xB)$ satisfies the gradient Lipschitz conditions:
\begin{align}
    \text{sup}_{\xB} \norm{\down_{\thetaB} \ell(\thetaB, \xB) -\down_{\thetaB} \ell(\thetaB^{*}, \xB) }_{2}  &\leq L_{\thetaB\thetaB} \norm{\thetaB - \thetaB^{*}}_{2} \nonumber\\
    \text{sup}_{\thetaB} \norm{\down_{\thetaB} \ell(\thetaB, \xB) -\down_{\thetaB} \ell(\thetaB, \xB^{*}) }_{2}  &\leq L_{\thetaB\xB} \norm{\xB - \xB^{*}}_{2} \nonumber\\
    \text{sup}_{\xB} \norm{\down_{\xB} \ell(\thetaB, \xB) -\down_{\xB} \ell(\thetaB^{*}, \xB) }_{2}  &\leq L_{\xB\thetaB} \norm{\thetaB - \thetaB^{*}}_{2} \nonumber
\end{align}
\end{assumption}
where $L_{\thetaB\thetaB}$, $L_{\thetaB\xB}$, and $L_{\xB\thetaB}$ are positive constants. Assumption~\ref{assum:continuous_paper} was used in \citep{sinha2018certifiable,wang19i}.

\begin{assumption}
\label{assum:strongconvex_paper}
The function $\ell(\thetaB, \xB)$ is \textit{locally $\mu$-strongly concave} in $\Xcal=\{\xB\,:\, \norm{\xB-\xB^{0}_{i}}_{\infty} \leq \epsilon \}$ for all $i \in [n]$, \ie, for any $\xB_{1}\, , \, \xB_{2} \in \Xcal_{i}$:
\begin{equation}
    \ell(\thetaB, \xB_{1}) \leq \ell(\thetaB, \xB_{2}) + \langle \down_{\xB}\ell(\thetaB, \xB_{2}), \xB_{1}-\xB_{2} \rangle - \frac{\mu}{2} \norm{\xB_{1}-\xB_{2}}^2_{2} \nonumber
\end{equation}
\end{assumption}
where $\mu$ is a positive constant which measures the curvature of the loss function. This assumption was used for analyzing distributional robust optimization problems \citep{sinha2018certifiable}.

\begin{assumption}
\label{assum:stochastic_variance_paper}
The variance of the stochastic gradient $g(\thetaB)$ is bounded by a constant $\sigma^2>0$:
\begin{equation}
   \Ebb[\norm{g(\thetaB)-\down L(\thetaB)}^{2}_{2}] \leq \sigma^2 \nonumber
\end{equation}
where $\down L(\thetaB) = \down_{\thetaB} \left(\sum_{i=1}^{N} \ell(\thetaB, \xB^*_i)\right)$ is the full gradient.
\end{assumption}
The Assumption~\ref{assum:stochastic_variance_paper} is commonly used for analyzing stochastic gradient optimization algorithms.

We denote the objective function in Equation~\ref{minmax_obj} as $L(\thetaB)$, its gradient by $\down L(\thetaB)$, the optimality gap between the initial neural network parameters and the optimal neural network parameters $\Delta=L(\thetaB_{0})-\min_{\thetaB}L(\thetaB)$,  the maximum distance between the output adversarial example generated by Amata and the original example as $\delta$, and $T$ as the number of iterations. Then, we have the following theorem for convergence of the algorithm:

\begin{theorem}[Convergence of Amata]\label{thm:convergence},
If the step size of outer minimization is $\eta_{t}=\min(1/\beta, \sqrt{\frac{\Delta}{TL\sigma^2}})$. Then, after $T$ iterations, we have:
\begin{equation}
    \frac{1}{T} \sum_{t=0}^{T-1} \Ebb[\norm{\down L(\thetaB_{t})}^{2}_{2}] \leq 4\sigma \sqrt{\frac{\beta\Delta}{T}} + 5L_{\thetaB\xB}^{2}\delta^2 \nonumber
\end{equation}
where $\sigma$ is the bound for variance between the batch gradient and the stochastic gradient and $\beta=L_{\thetaB\xB}L_{\xB\thetaB}/\mu + L_{\thetaB\thetaB}$ is a constant. 
\end{theorem}
The detailed notations, assumptions and proof are shown in the Appendix due to limited space. This theorem proves that our algorithm can converge under a suitable selection of learning rate.

The remainder of this section serves to motivate and justify, both qualitatively and quantitatively, the annealing method in Amata from an optimal control viewpoint. We will try to answer this question in Section~\ref{sec:optimal_control_fomulation},~\ref{sec:landscape_analysis},~\ref{sec:criterion}:

\emph{How good is the annealing scheduling in terms of adversarial training acceleration?}

We start with a general formulation of annealed adversarial training as an optimal control problem.

% In the following section, we develop a general formulation of annealed or adaptive adversarial training based on optimal control theory and derive a novel criterion to quantify the optimality of an annealing strategy, taking into account the trade-off between accuracy and efficiency. Furthermore, we show that Amata performs favorably under this criterion, thus providing further justification for our approach.

\subsection{Optimal control formulation of annealed adversarial training}
\label{sec:optimal_control_fomulation}
In essence, the PGD-based adversarial training algorithm \cite{madry2017towards} is a result of a number of relaxations of the original min-max problem in Eq.~\eqref{minmax_obj}, which we will now describe.
For simplicity of presentation, let us consider just one fixed input-label pair $(\xB^{0}, y)$, since the $N$-sample case is similar. The original min-max adversarial training problem is given in~\eqref{minmax_obj} with $N=1$.
% \begin{align}\label{eq:adv}
%     \min_{\thetaB}
%     \max_{ \{\zB : \| \zB - \xB^{0} \| \leq \epsilon \} }
%     \ell(h_{\thetaB}(\zB), y).
% \end{align}
The first relaxation is to replace the outer minimization with gradient descent so that we obtain the iteration
\begin{align}
    \thetaB_{t+1} = \thetaB_{t} - \eta \nabla_{\thetaB}
    \max_{ \{\xB : \| \xB - \xB^{0} \| \leq \epsilon \} }
    \ell(h_{\thetaB_t}(\xB), y).
\end{align}
Then, the remaining maximization in each outer iteration step is replaced by an abstract algorithm
$
    \Acal_{\uB, \thetaB} : \Rcal^d \rightarrow \Rcal^d
$
which solves the inner maximization approximately. Here, we assume that the algorithm depends on the current parameters of our neural network $\thetaB$, as well as hyper-parameters $\uB$ which takes values in a closed subset $G$ of an Euclidean space. No further assumptions are placed on $G$, which may be a continuum, a countable set, or even a finite set.

This relaxation leads to the following iterations\footnote{Here we assume that the gradient with respect to $\theta$ is the partial derivative with respect to the parameters of the network $h_\theta$ and $\theta_t$ in $\Acal_{\theta_t,\uB_t}$ is held constant. This is the case for the PGD algorithm. Alternatively, we can also take the total derivative, but this leads to different algorithms. }
\begin{equation}\label{eq:general_relaxation}
    \thetaB_{t+1} = \thetaB_{t} - \eta \nabla_{\thetaB}
    \ell(h_{\thetaB_t}(\Acal_{\thetaB_t,\uB_t} ), y).
\end{equation}
Eq.~\eqref{eq:general_relaxation} represents a general formulation of annealed adversarial, of which Algorithm~\ref{algo:aal} is an example with $\Acal_{\theta_t,u_t}$ being the inner PGD loop and $\uB_t = \{ \alpha_t, K_t \}$ are the hyper-parameters we pick at each $t$ step.
% If all of these are fixed in $t$, we obtain the original PGD-based adversarial training algorithm. In fact, the PGD-based adversarial training algorithm is a special case of~\eqref{eq:general_relaxation} where $\uB_t = \uB$ for all $t$, and $\Acal_{\uB, \thetaB_t}$ represents PGD iterations at model parameter $\thetaB_t$, whose hyper-parameters (learning rate, number of steps) are fixed at $\uB$. An \emph{annealed} version of the algorithm then allows for $\uB_t$ to vary in $t$, which indexes the outer loop GD iterations.
The function $t\mapsto \uB_t$ is an \emph{annealing schedule} for the hyper-parameters. How to pick an optimal schedule can be phrased as an optimal control problem.

To make analysis simple, we will take a continuum approximation assuming that the outer loop learning rate $\eta$ is small. This allows us to replace~\eqref{eq:general_relaxation} by an ordinary differential equation or gradient flow with the identification $s\approx t\eta$:
\begin{equation}\label{eq:general_relaxation_cts}
    \dot{\thetaB}_{s} = - \nabla_{\thetaB}
    \ell(h_{\thetaB_s}(\Acal_{\thetaB_s,\uB_s} ), y).
\end{equation}
Here, the time $s$ is a continuum idealization of the outer loop iterations on the trainable parameters in the model.
We consider two objectives in designing the annealing algorithm: on a training interval $[T_1,T_2]$ in the outer loop, we want to minimize the loss under adversarial training measured by a real-valued function $\Phi(\thetaB)$ while also minimizing the training cost associated with each inner algorithm loop under the hyper-parameter $\uB$, which is measured by another real-valued function $R(\uB)$. An optimal annealing algorithm can then be defined as a solution to the following problem:
\begin{equation}\label{eq:cts_control_problem}
    \begin{aligned}
        &\min_{\uB_{T_1:T_2}}
        \Phi(\thetaB_T)
        + \int_{T_1}^{T_2} R(\uB_s) ds
        \\
        &\text{subject to:} \quad \dot{\thetaB}_{t} = F(\thetaB_s, \uB_s)\\
        &\quad \text{and} \quad
        F(\thetaB_s, \uB_s) := - \nabla_{\thetaB}
        \ell(h_{\thetaB_s}(\Acal_{\thetaB_s,\uB_s} ), y),
    \end{aligned}
\end{equation}
where we have defined the shorthand $\uB_{T_1:T_2} = \{\uB_{s}:s\in[T_1,T_2]\}$. In this paper, we take $\Phi$ to be the DNN's prediction loss given an adversarial example (adversarial robustness), and set $R(\uB_s) = \gamma K_s$, where $K_s$ is the number of inner PGD steps at outer iteration number $s$, and $\gamma$ is the coefficient for trade-off between adversarial robustness and training time. This is to account for the fact that when $K_s$ increases, the cost of the inner loop training increases accordingly. The integral over $s$ of $R$ is taken so as to account for the total computational cost corresponding to a choice of hyper-parameters $\{ \uB_s \}$. The objective function taken as a sum serves to balance the adversarial robustness and computational cost, with $\gamma$ as a balancing coefficient.

Problem~\eqref{eq:cts_control_problem} belongs to the class of Bolza problems in optimal control, and its necessary and sufficient conditions for optimality are well-studied.
In this paper, we will use a necessary condition, namely the Pontryagin's maximum principle, in order to motivate our annealing algorithm and derive a criterion to test its approximate optimality. For more background on the theory of calculus of variations and optimal control, we refer the reader to~\cite{boltyanskii1960theory,bertsekas1995dynamic}.

% where $\Phi(\theta) :=  \E_{(x,y)\sim \mu} \ell(h_{\theta}(x), y)$ (or we can also have an inner perturbation loop at the end...) and $R$ denotes a running cost, e.g. $R(\alpha, \theta) = \gamma/\alpha$ (penalizing small $\alpha$'s) where $\gamma$ is the hyper-parameter determining the relative importance of time cost.

% \subsection{Pontryagin's Maximum Principle}

% In the last section, the problem of choosing hyper-parameters in the inner loops of adversarial training has been formulated as an optimal control problem in~\eqref{eq:cts_control_problem}. Now, we show how this connection can help us design and validate algorithms.

% A classical result in calculus of variations gives the following necessary conditions for optimality.
\begin{theorem}[Pontryagin's Maximum Principle (PMP)]\label{thm:pmp}
    Let $\uB^*_{T_1:T_2}$ be a solution to~\eqref{eq:cts_control_problem}. Suppose $F(\thetaB, \uB)$ is Lipschitz in $\thetaB$ and measurable in $\uB$.
    Define the Hamiltonian function
    \begin{align}
        H(\thetaB, \pB, \uB) = \pB^\top F(\thetaB, \uB)
        - R(\uB)
    \end{align}
    Then, there exists an absolutely continuous co-state process $\pB^*_{T_1:T_2}$ such that
    \begin{align}
        &\dot{\thetaB}^*_s = F(\thetaB^*_s, \uB^*_s)
        \qquad\qquad\qquad \thetaB^*_{T_1} = \thetaB_{T_1} \label{eq:pmp_state} \\
        &\dot{\pB}^*_s = - \nabla_{\thetaB}
        H(\thetaB^*_s, \pB^*_s, \uB^*_s)
        \qquad \pB^*_{T_2} = - \nabla_{\thetaB} \Phi(\thetaB^*_{T_2}) \label{eq:pmp_costate} \\
        &H(\thetaB^*_s, \pB^*_s, \uB^*_s) \geq H(\thetaB^*_s, \pB^*_s, \vB)
        \quad \forall
        \vB \in G, s\in[T_1,T_2] \label{eq:pmp_maxh}
    \end{align}
\end{theorem}
In short, the maximum principle says that a set of optimal hyper-parameter choices $\{ \uB^*_s \}$ (in our specific application, these are the optimal choices of inner-loop hyper-parameters as the training proceeds) must globally maximize the Hamiltonian defined above for \emph{each} outer iteration. The value of the Hamiltonian at each layer depends in turn on coupled ODEs involving the states and co-states.
This statement is especially appealing for our application because unlike first-order gradient conditions, the PMP holds even when our hyper-parameters can only take a discrete set of values, or when there are non-trivial constraints amongst them.

We now show how the optimal control formulation and the PMP motivates and justifies the proposed Amata both qualitatively and quantitatively.

% We now show that it gives us a quantitative measure of deviation from optimality, which we can use to validate algorithms.

% \subsection{Motivating example for Amata}

\begin{figure}[!ht]
\begin{center}
\setlength\tabcolsep{0.01pt}
\begin{tabular}{cc}
\includegraphics[width=0.48\columnwidth]{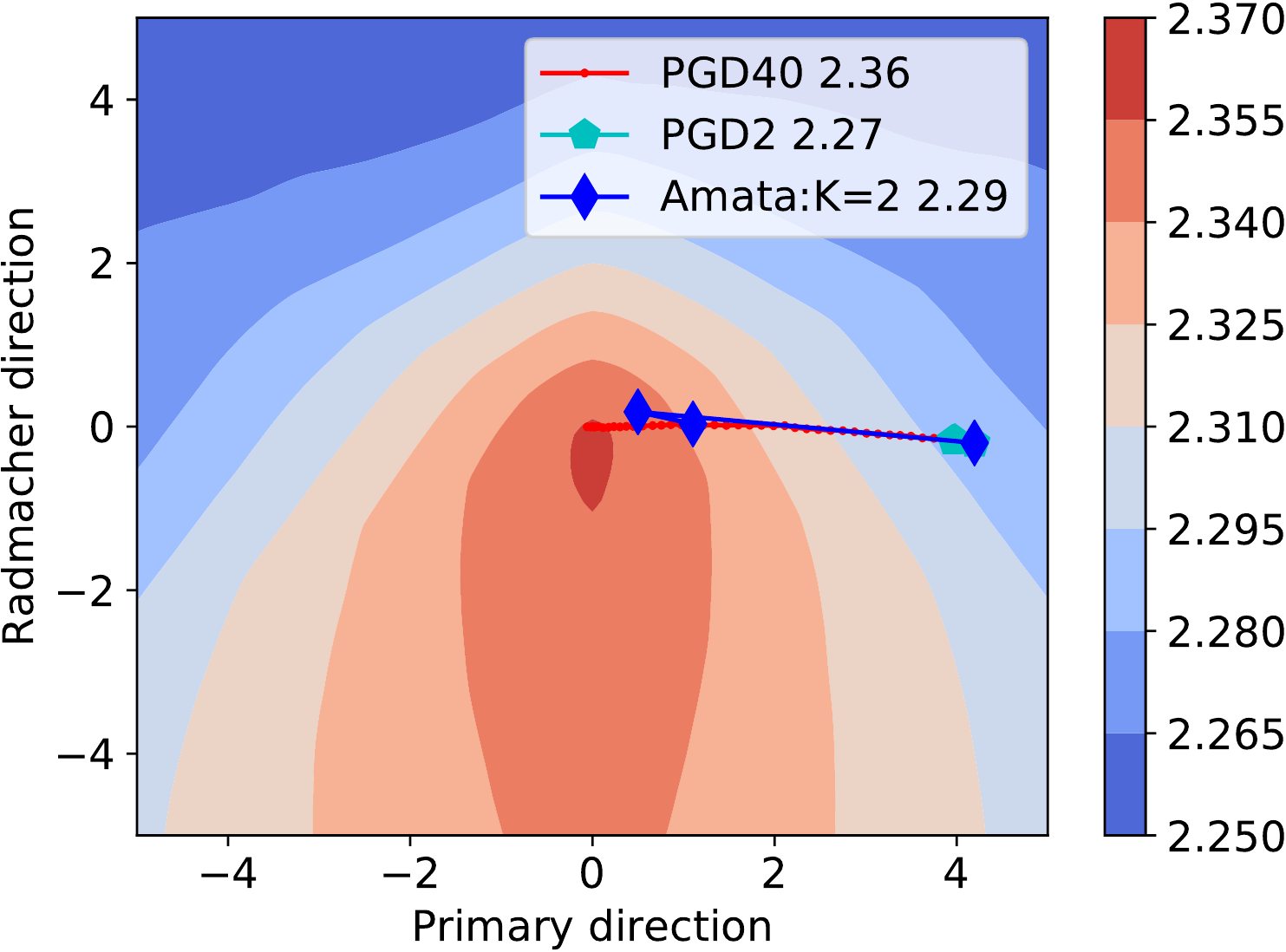} &
\includegraphics[width=0.48\columnwidth]{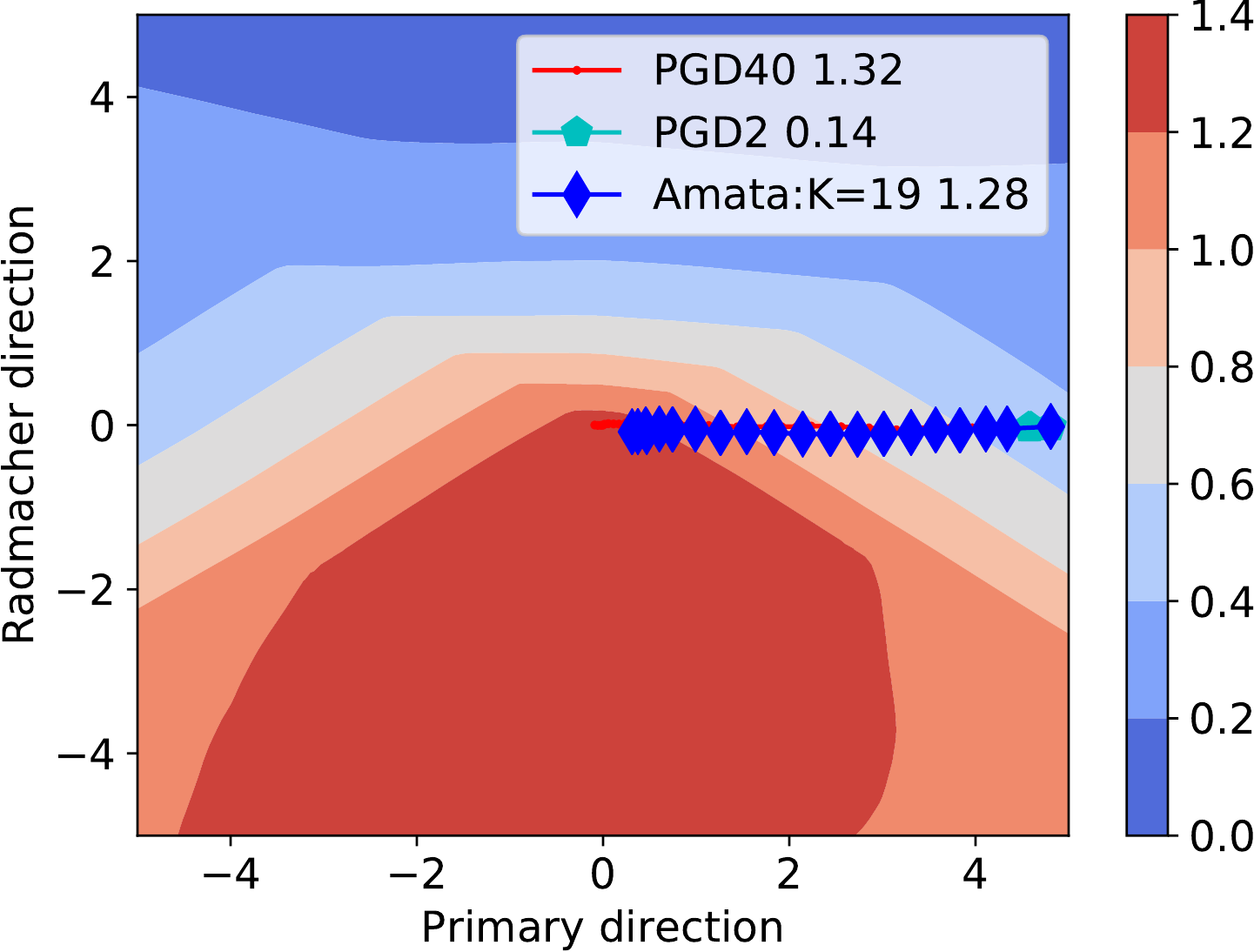} \\
\end{tabular}
\end{center}
\caption{Visualization of inner maximizations landscape and PGD-40, PGD-2, and Amata's trajectories at Epoch 1 (\textbf{Left}) and Epoch 10 (\textbf{Right}). $K$ is the Amata's number of steps and the numerics indicate the adversarial loss obtained by different methods (the higher the better for adversarial perturbations).
% We can see that at the beginning of adversarial training (Epoch 1), the adversarial loss landscape is smooth (difference between the maximum and the minimum is small). It is easy for all methods to achieve similar adversarial losses. As the adversarial training proceeds, loss landscape is steeper and requires more iterations at smaller step sizes to get a good solution.
Applying standard PGD-2 for acceleration cannot find strong adversarial example at Epoch 10. However, Amata can adaptively balance the number of steps and the step size to achieve good trade-off between time costs and robustness as justified before by theoretic analysis. (Better viewed in the zoom-in mode). More details are in the Appendix.
}
\label{fig:vis_trajectory_main}
\end{figure}

\subsection{Qualitative justification of Amata via landscape analysis}
\label{sec:landscape_analysis}

We now show that the form of an effective annealing schedule depends on how the loss landscape with respect to the data (i.e. landscape for the inner loop maximization) changes as training proceeds. Let us first visualize the change in landscape for actual neural network training using a similar method as \cite{Ali2019}. In Figure~\ref{fig:vis_trajectory_main}, we observe that as the outer training proceeds, the loss landscape with respect to the data becomes steeper, and thus we require smaller step sizes and/or more steps in the inner maximization to find good quality adversarial examples.

Let us now make this observation more precise using the maximum principle. In general, it is not possible to solve the PMP Eq.~(\ref{eq:pmp_state}-\ref{eq:pmp_maxh}) in closed form. However, we can consider representative loss landscapes qualitatively similar to Fig.~\ref{fig:vis_trajectory_main} where such equations are solvable in order to motivate our proposed algorithm. To this end, consider a 1D optimization problem with loss function
\begin{equation}
    \ell(\theta,x) = \frac{\theta ^2}{2}-\frac{(x-\theta )^2}{\theta ^2+1},
\end{equation}
where $x$ plays the role of data and $\theta$ plays the role of
trainable parameters. We will assume that the data point $x=0$ so that the non-robust loss is
$
    l(\theta, 0) = \frac{\theta ^2}{2}+\frac{1}{\theta ^2+1}-1,
$
which has two minima at $\theta = \pm \sqrt{\sqrt{2} - 1}$. However, the robust loss is
$
    \tilde{\ell}(\theta) = \max_{x \in \Rcal} l(\theta, x)
    = \frac{1}{2} \theta^2,
$
which has a unique minimum at $\theta=0$.
The key motivation behind this example is the fact that the loss landscape for $x$ becomes more steep (i.e. $|d\ell/dx|$ increases) as training in $\theta$ proceeds towards the unique minimum of the robust loss at $0$, just like in actual adversarial training in Figure~\ref{fig:vis_trajectory_main}.
Effectively, this means that as training proceeds, it becomes more important to ensure the stability/convergence of the inner loop training. Given a limited computational budget, one should then allocate more of them towards the later part of training, hence giving rise to an annealing schedule that gradually increases the complexity of the inner loop iterations. Indeed, this can be made precise for this example - if we apply two-loop adversarial training with $K_t$ inner loop steps at step size $\alpha_t$ (with $\tau = \alpha_t K_t$ fixed, c.f. Algorithm~\ref{algo:aal}) to $\ell$, we can solve the PMP explicitly in this case to obtain the optimal schedule for $K_t$.
\begin{equation}\label{eq:toy_soln}
    {K_t^*}= \frac{\tau}{{\alpha_t^*}}
    = \frac{2\tau}{\theta_0^2 e^{-2t}+1}
    = \frac{2 \tau }{\theta_0 ^2+1}
    +
    \left(
	    \frac{4 \theta_0 ^2 \tau }{\left(\theta_0 ^2+1\right)^2}
    \right)
    t + \mathcal{O}(t^2),
\end{equation}
where the last step is valid for small $t$. This is a schedule that gradually increases the number of inner maximization steps as the outer loop $t$ proceeds, and  motivates our annealing choice in Amata. Note that this example is qualitative, so we do not adopt an identical schedule in Amata, but a linear approximation of it (right hand side in~\eqref{eq:toy_soln}) that also increases in time. It is clear from the derivation that the origin of this increasing schedule is that the curvature $l_{xx}(\theta, x)$ increases as $\theta\rightarrow \theta^* = 0$, hence this motivates the Amata algorithm in view of the numerical observations in Fig~\ref{fig:vis_trajectory_main}.

\subsection{Quantitative justification of Amata via a criterion of approximate optimality}
\label{sec:criterion}

Besides qualitative motivations, it is desirable to have a criterion to test the approximate optimality of any chosen annealing schedule. In this subsection, we develop a quantitative criterion for this purpose, based on the PMP. Given any hyper-parameter choice $\uB_{T_1:T_2}$ over the training interval, let us define its ``distance'' from optimality as
\begin{align}\label{eq:criterion}
    C(\uB_{T_1:T_2}) &= \frac{1}{T_2 - T_1} \int_{T_1}^{T_2}
        \max_{\vB \in G} H(\thetaB^{\uB}_s, \pB^{\uB}_s, \vB)
        \nonumber \\ &-
        H(\thetaB^{\uB}_s, \pB^{\uB}_s, \uB_s)
    ds
\end{align}
where $\{ \thetaB^{\uB}_s, \pB^{\uB}_s : s\in[T_1, T_2]\}$ represents the solution of the Eq.~\eqref{eq:pmp_state} and~\eqref{eq:pmp_costate} with $\uB_s$ in place of $\uB^*_s$.
Observe that $C(\uB_{T_1:T_2}) \geq 0$ for any $\uB_{T_1:T_2}$ with equality if and only if $\uB_{T_1:T_2}$ satisfies the PMP for almost every $s\in[T_1,T_2]$. Hence, $C$ can be used as a measure of deviation from optimality. When $C$ is small, our annealing strategy $\{ \uB_s \}$ is close to at least a locally optimal strategy, where as when it is large, our annealing strategy is far from an optimal one. For ease of calculation, we can further simplify $C$ by Taylor expansions, assuming $T_2-T_1=\eta$ is small, yielding
% \paragraph{One-step Approximation and Adversarial Training Criterion.}
% Equation~\eqref{eq:criterion} requires information on the entire training interval and may be expensive to compute. To save computation, we use a one-step approximation where we take $T_1=t$ (current iteration) and $T_2 = t + \eta$ with $\eta \ll 1$. In this small interval, we can also take $\uB_s$ to be constant and thus equal to some $\uB$. This is in some sense a greedy approximation, where we evaluate the short term optimality of a piece-wise constant choice of hyper-parameters. From~\eqref{eq:criterion} we then obtain via a Taylor expansion and our particular choices of $\Phi$ and $R$
(See Appendix~\ref{sec:app_criterion})
\begin{equation}\label{eq:simplified_criterion}
    \begin{aligned}
        C(\uB_{t:t+\eta})
        % &=
        % C(\uB_t, t) + o(1) ~ \text{ with} ~
        % C(\uB_t, t) \equiv C(\alpha_t, K_t, t)\\
        &\approx
        \max_{\alpha, K}
        \left\{
            \| \nabla_{\thetaB} \ell ( h_{\thetaB_{t}}
                [\Acal_{\thetaB_t,\alpha,K}(x)], y)\|^2 - \gamma K
        \right\}
        \\ &-
        \left(
            \| \nabla_{\thetaB} \ell ( h_{\thetaB_{t}}
            [\Acal_{\thetaB_t,\alpha_t, K_t}(x)], y)\|^2 - \gamma K_t
        \right)
    \end{aligned}
\end{equation}
with $\Acal_{\thetaB_t,\alpha,K}$ denoting the inner PGD loop starting from $x$ with $K$ steps and step size $\alpha$. Criterion~\eqref{eq:simplified_criterion} is a greedy version of the general criterion derived from the maximum principle. It can be used to either evaluate the near-term optimality of some choice of hyper-parameters $\uB$, or to find an approximately optimal hyperparameter greedily by solving $C(\uB,t) = 0$ for $\uB$, which amounts to maximizing the first term.
In this paper, we use Bayesian optimization\footnote{Implementations can be found in https://github.com/hyperopt/hyperopt} to perform the maximization in~\eqref{eq:simplified_criterion} to evaluate strategies from the controllable space $G$.

\paragraph{Comparison with FOSC criterion:}
\cite{wang19i} proposed an empirical criterion to measure the convergence of inner maximization:
\begin{equation}
    \text{FOSC}(\xB)
    =
    \epsilon \norm{\nabla_{\xB} \ell(h_{\thetaB}(\xB), y)}
    -
    \langle
        \xB -\xB^{0},  \nabla_{\xB} \ell(h_{\thetaB}(\xB), y)
    \rangle
\end{equation}

%There are some similarities between our criterion and FOSC when we do not consider the computational cost term $R$. For example, when the stationary saddle point is achieved, both our criterion and FOSC reach the minimum. However, our proposed criterion is quite different from FOSC in the following aspects. First, our criterion is derived from optimal control theory, whereas FOSC is concluded from empirical observations. Second, our criterion takes computation costs into consideration, whereas FOSC only considers the convergence of adversarial training. Lastly, our criterion is based on the gradient of DNN parameters, whereas FOSC is based on the gradient of the input. Measuring the gradient of DNN parameters is arguably more suitable for considering robustness-training time trade-off as the variance of the DNN parameters is much larger than the input during training.

There are some similarities between the proposed criterion and FOSC when the computational cost term $R$ is not considered. For example, when the stationary saddle point is achieved, both the proposed criterion and FOSC reach the minimum. However, the proposed criterion is different from FOSC in the following aspects:
1) The proposed criterion is derived from the optimal control theory, whereas FOSC is concluded from empirical observations. 2) The proposed criterion takes computation costs into consideration, whereas FOSC only considers the convergence of adversarial training. 3) The proposed criterion is based on the gradient of DNN parameters, whereas FOSC is based on the gradient of the input. Measuring the gradient of DNN parameters is arguably more suitable for considering robustness-training time trade-off as the variance of the DNN parameters is much larger than the input during training.

% \begin{enumerate}
%     \item Our criterion is derived from the optimal control theory, whereas FOSC is concluded from empirical observations.
%     \item Our criterion takes computation costs into consideration, whereas FOSC only considers the convergence of adversarial training.
%     \item Our criterion is based on the gradient of DNN parameters, whereas FOSC is based on the gradient of the input. Measuring the gradient of DNN parameters is arguably more suitable for considering robustness-training time trade-off as the variance of the DNN parameters is much larger than the input during training.
% \end{enumerate}

\subsection{Evaluation of Amata using $C$.}
We now use the numerical form of the optimal control criterion~\eqref{eq:simplified_criterion} to analyze Amata for robustness and computational efficiency trade-off. We use the LeNet architecture \footnote{Implementations can be found in https://github.com/pytorch/examples/blob/master/mnist/main.py} for MNIST classification as an example. We set the $\gamma$ as 0.04 for this criterion, and show the result in Table~\ref{tab:numeric_comparison}. From this Table, we observe that $C$ and performance (robustness, time) are correlated in the expected way, and that Amata has lower $C$ values and better performances. Furthermore, $C$ takes into account both robustness and computational cost as seen in the first two rows, where a lower $C$ value is associated with similar robustness but lower time cost. Hence, $C$ can help us choose a good annealing schedule.

Although computing the exact optimal control strategy for DNN adversarial training is expensive for real-time tasks, with the criterion derived from the PMP, we are able to numerically compare the optimality of different adversarial training strategies. From this numerical evaluation, we have demonstrated that the proposed Amata algorithm is close to an optimal adversarial training strategy, or at least one that satisfies the maximum principle. We will show that our algorithm can achieve similar or even better adversarial accuracy much faster with empirical experiments on popular DNN models later in Experiments section.

\vspace{-0.3cm}
\begin{table}[!ht]
\centering
\captionof{table}{Comparison of adversarial training strategies. Amata setting 1: $K_{min}=5,K_{max}=40$, Amata setting 2: $K_{min}=10,K_{max}=40$.}
    \begin{tabular}{cccc}\hline
        Strategy                    & C  &Robustness &Time \\ \hline
        \textit{Amata(Setting 1)} & \textit{\textbf{0.54}}     & \textit{\textbf{91.47\%}}   &\textit{697.73s}\\
        \textit{Amata(Setting 2)} & \textit{\textbf{0.68}}   & \textit{\textbf{91.46\%}}   &\textit{760.16s} \\
        \hline
        PGD-10                      &7.82      & 68.07\%   &307.57s\\
        PGD-20                      &1.52       & 85.23\%   &567.11s\\
        PGD-40                      &1.20       & 90.56\%   &1086.31s\\ \hline
      \end{tabular}
       % add robustness training time
\label{tab:numeric_comparison}
\end{table}

\section{Experiments}
\label{sec:exp}
To demonstrate the effectiveness of Amata mechanism, experiments are conducted on MNIST(in the appendix), CIFAR10, Caltech256 and the large-scale ImageNet dataset. With less computational cost, the proposed Amata method trained models achieve comparable or slightly better performance than the models trained by other methods, such as PGD. In our experiment, PyTorch 1.0.0 and a single GTX 1080 Ti GPU were used for MNIST, CIFAR10, and Caltech256 experiment, while PyTorch 1.3.0 and four V100 GPUs were used for the ImageNet experiment. \textcolor{black}{Note that different from research focus on improving adversarial accuracy, for efficient adversarial training research, it is important to run algorithms in the same hardware and driver setting for fair comparison.}
We evaluate the adversarial trained networks against PGD and Carlini-Wagner (CW) attack \cite{carlini2017towards}. In addition, Amata can also be seamlessly incorporated into existing adversarial training acceleration algorithms, such as you only propogate once(YOPO, \cite{zhang2019you}), adversarial training for free (Free, \cite{Ali2019}), and fast adversarial training (Fast, \cite{Wong2020Fast}). As an ablation study, results with other annealing schemes, such as exponential one, are shown in the Appendix. We first evaluate Amata on standard datasets and then incorporate Amata into other adversarial training acceleration algorithms.

\subsection{Evaluation of Amata on standard datasets}

\paragraph{CIFAR10 classification}
For this task, we use the PreAct-Res-18 network \cite{madry2017towards}. PGD, FOSC, \textcolor{black}{FAT}, and Amata are tested. $\tau$ is set as 20/255 for Amata. The clean and robust error of PGD-10 and Amata are shown in Figure~\ref{fig:timeerror_curve} Left. The proposed Amata method takes 3045 seconds to achieve less than 55\% robust error while for PGD-10, it takes 6944 seconds. FOSC takes 8385 seconds to achieve similar accuracy. \textcolor{black}{During the experiment, FAT cannot achieve 55\% robust error. This is because FAT always generate adversarial examples near the decision boundaries and the adversarial loss might be too small to make the adversarial training effective.} Furthermore, we run the PGD, FOSC and Amata experiment for 100 epochs until full convergence, with showing the clean accuracy, PGD-20 attack accuracy, CW attack accuracy, and the consumed time in Table~\ref{table:cifar10}. We can see that Amata outperforms PGD-10 with reducing the consumed time to 61.9\% and achieves comparable accuracy to FOSC with reducing the consumed time to 54.8\%.

\vspace{-1cm}
\begin{table*}[!ht]
    \caption{CIFAR10 adversarial training convergence results. We run all algorithms on the same computation platform for fair comparison.
    }
    \label{table:cifar10}
    \centering
    \small
    \begin{tabular}{c|c|c|c|c}
    \hline
    Training methods & Clean accuracy	 &PGD-20 Attack  &CW Attack &Time (Seconds)\\
    \hline
    ERM    & 94.75\%          &0.0\%         &  0.23\%    &2099.58\\
    \hline
    PGD-2         & 90.16\%            &31.70\%          &  13.36\%   &6913.36\\
    PGD-10         &  85.27\%            &47.31\%        &  51.73\%   &23108.10\\
    \hline
    FAT\citep{zhang2020attacks} &89.30\% &41.34\%       & 41.16\%      &14586.08\\
    \hline
    FOSC\citep{wang19i} &85.29\%          &47.75\%       &47.70\%      &26126.98\\
    \hline
    \textit{Amata($K_{min}=2$, $K_{max}=10$)}             & \textit{85.52\%}            &\textit{47.62\%}     & \textit{52.94\%}  &  \textit{14308.96}\\
    \hline
    \end{tabular}
\end{table*}
\vspace{-1cm}

\begin{figure}[!ht]
\begin{center}
\setlength\tabcolsep{0.2pt}
\begin{tabular}{cc}
\includegraphics[width=0.5\columnwidth]{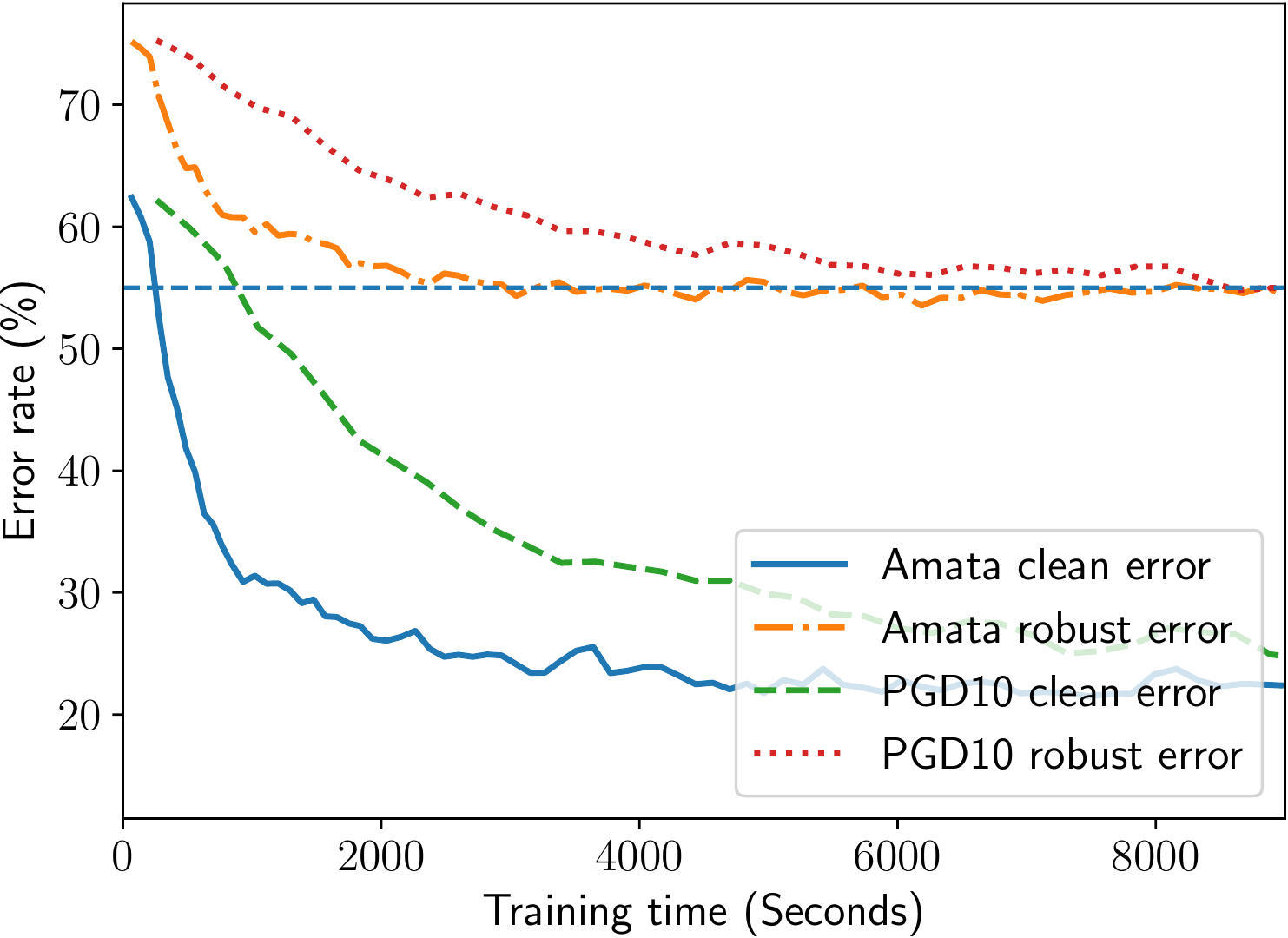} &
\includegraphics[width=0.5\columnwidth]{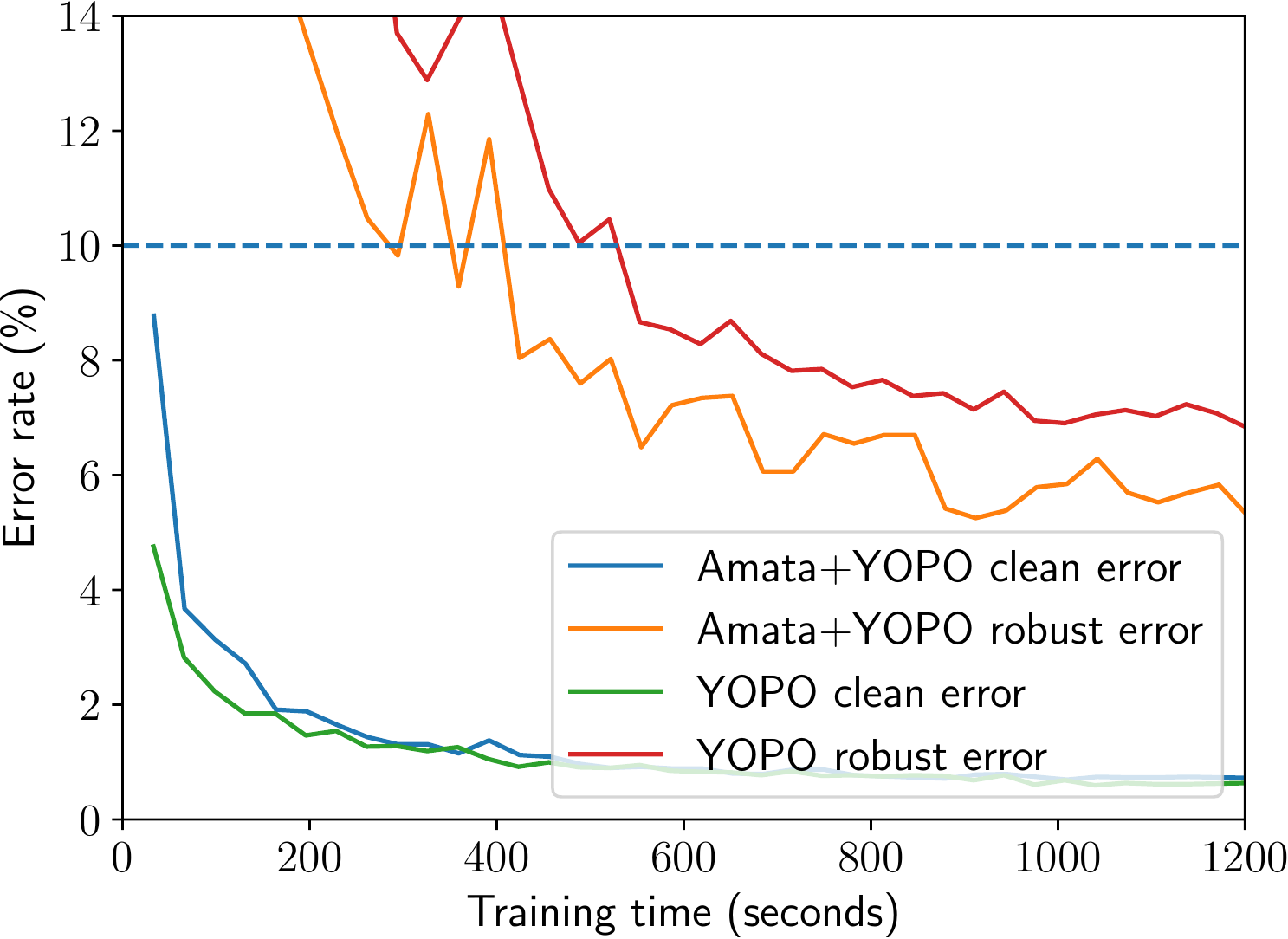} \\
\end{tabular}
\end{center}
\caption{\textbf{Left:} The clean and robust error of Amata and PGD-10 in the CIFAR10 validation for achieving less than 55\% robust error. We use Amata with the setting $K_{min}=2$ and $K_{max}=10$. \textbf{Right:} The clean and robust error of training with YOPO and Amata+YOPO on MNIST \footnote{We found the YOPO's code for CIFAR10 experiment is less stable to random seed}. We use Amata with the setting $K_{min}=2$ and $K_{max}=5$. (Better viewed in the zoom-in mode)}
\label{fig:timeerror_curve}
\end{figure}

\paragraph{Caltech256 classification}
Experiment is also conducted on Caltech256 dataset \cite{Caltech256} with ResNet-18 network. It consists of 30,607 images spanning 257 object categories. Object categories are diverse, including animals, buildings and musical instruments. We use the same experiment setting as in \cite{ZhangZ19}. As it is already known that FOSC cannot lead to acceleration, we do not include FOSC in this experiment. For achieving the adversarial accuracy of 28\%, Amata takes 3403 seconds while PGD-20 takes 4956 seconds. Furthermore, we run the PGD-20 and Amata training for 21 epochs until full convergence. The accuracies under clean, PGD-5, PGD-20, and PGD-100 attack data are shown in Table~\ref{table:Caltech256_ImageNet}. We can see that, with saving around 30\% computational time, the proposed Amata can achieve similar accuracy to PGD-20 under up to 100 iterations of PGD attacks.

\vspace{-0.1cm}
\begin{table*}[!ht]
    \caption{Caltech256 and ImageNet results. Amata and Amata+ achieve \emph{almost the same} robustness under various strengths of attacks after convergence.}
    \vspace{-0.3cm}
    \label{table:Caltech256_ImageNet}
    \centering
    \small
    \begin{tabular}{c|c|c|c|c}
    \hline
    \multicolumn{5}{c}{Caltech256 results.} \\
    \hline
    Training methods & Clean accuracy	 &PGD-5   &PGD-20  &PGD-100 \\
    \hline
    ERM    & 83.1\%          &0.0\%         &  0.0\%   &0.0\%\\
    \hline
    PGD-20        & 65.7\%            &29.7\%          &  28.5\%   &28.5\\
    \hline
    \textit{Amata($K_{min}=10$, $K_{max}=20$)}             & \textit{66.1\%}            &\textit{29.6\%}     & \textit{28.3\%} & \textit{28.3}\\
    \hline
    \multicolumn{5}{c}{ImageNet results.} \\
    \hline
    Training methods & Clean accuracy	 &PGD-10   &PGD-20  &PGD-50 \\
    \hline
    Free    &60.57\%    &32.1\%   &31.5\%   &31.3\%\\
    \hline
   \textit{Amata($K_{min}=2$, $K_{max}=4$)+Free}             & \textit{59.7\%}            &\textit{31.8\%}     & \textit{31.2\%} & \textit{31.0}\\
    \hline
    \end{tabular}
\end{table*}

\subsection{Amata+}
\label{AmataYOPO}
The proposed Amata mechanism is largely orthogonal to existing acceleration approaches to adversarial training acceleration, and hence can be readily incorporated into them. We now demonstrate this for YOPO~\cite{zhang2019you},  adversarial training for free (Free)~\cite{Ali2019}, and fast adversarial training (Fast)~\cite{Wong2020Fast}. We name this kind of jointly implemented Amata as Amata+.

\vspace{-0.1cm}
\paragraph{Amata+YOPO}
YOPO's MNIST classification experiment\footnote{https://github.com/a1600012888/YOPO-You-Only-Propagate-Once/tree/82c5b902508224c642c8d0173e61435795c0ac42/\\experiments/MNIST/YOPO-5-10} is demonstrated as an example. For Amata incorporation, we gradually increase the $K$ and decrease the $\sigma$ in the codes that is similar to the case of modifying the PGD algorithm. The clean and robust error of YOPO and Amata+YOPO are shown in Figure~\ref{fig:timeerror_curve} Right. We can see that Amata+YOPO takes 294 seconds to reach the adversarial accuracy of 94\%, which is around the half of the time consumed by YOPO. It is also worth noting that Amata+YOPO achieves better adversarial accuracy when converged. This phenomenon corresponds to the finding in FOSC~\cite{wang19i} that too strong adversarial example is not needed at the beginning. From this example, we can see that Amata can be easily incorporated in other adversarial training acceleration algorithm to provide further acceleration and to improve adversarial accuracy.

\vspace{-0.1cm}
\paragraph{Amata+Free}
ImageNet is a large-scale image classification dataset consisting of 1000 classes and more than one million images
\cite{ILSVRC15}. Adversarial training on ImageNet is considered to be challenging due to the high computation cost \cite{Kannan2018,Cihang2018}. Recently, Ali \etal proposed the Free algorithm for adversarial training acceleration by alternatively running PGD-1 to solve the inner maximization and the outer minimization. This process is run $m=4$ times for each input data batch \footnote{https://github.com/mahyarnajibi/FreeAdversarialTraining}, with four V100 GPUs. For Amata incorporation, similarly, we increase $m$ from 2 to 4 and decrease the PGD step size also by two times in the code. In the experiment, ResNet-50 is used for adversarial training. We find that it takes Amata+Free 948 minutes to achieve 30\% adversarial accuracy which saves around 1/3 computational time compared with 1318 minutes by the Free algorithm. We further run Free and Free+Amata 22 epochs for full convergence and test the obtained models with various iterations of PGD attacks. The results are shown in Table~\ref{table:Caltech256_ImageNet}. We can see that Amata can still help reducing the computational cost almost without performance degradation even when combined with the state-of-the art adversarial training acceleration algorithm on the large-scale dataset.
\vspace{-0.1cm}
\paragraph{Amata+Fast}
We also incorporate Amata into the fast adversarial training algorithm by using a weaker version of the fast adversarial training algorithm at the initial stage and then using the original version of the fast adversarial training algorithm later. The weaker version of the adversarial training algorithm is constructed by using a fixed non-zero initialization at the start. With the same setting\footnote{\url{https://github.com/locuslab/fast_adversarial/}}, Amata+Fast can achieve 72\% PGD-20 accuracy two times faster than fast adversarial training on CIFAR10. This further shows that Amata is a mechanism that can be readily incorporated into many existing algorithms.
 
\vspace{-0.1cm}
\paragraph{Amata+ATTA}
We find that Amata can be seamlessly combined with a recently proposed adversarial training method---adversarial training with transferable adversarial examples (ATTA). We follow the same setting as in \cite{zheng2019efficient}. ATTA achieves efficient adversarial training by reusing a number of adversarial perturbations calculated in previous epochs, which is controled by a hyper-parameter $\text{reset}$ in \footnote{\url{https://github.com/hzzheng93/ATTA}}. To implement Amata, we reduce $\text{reset}$ to be 2 in the first five epochs to reduce the strength of adversarial examples. Compared with Amata, Amata+ATTA can achieve 58\% PGD-20 accuracy around 1.5 times faster than ATTA.

\section{Conclusion}
We proposed a novel annealing mechanism for accelerating adversarial training that achieves comparable or better robustness with 1/3 to 1/2 the computational cost over a variety of benchmarks. Moreover, a convergence proof and a general optimal control formulation of annealed adversarial training is developed to justify its validity and performance. Our approach can also be seamlessly incorporated into existing adversarial training acceleration algorithms, such as YOPO and Adversarial Training for Free to achieve acceleration and improve performance. As a point of future work, we will explore adaptive methods for adversarial training based on the optimal control formulation (e.g. the MSA algorithm~\cite{Chernousko1982a,Li2017,li2018optimal}).

\section{Acknowledgements}
Nanyang Ye was supported in part by National Key R\&D Program of China  2017YFB1003000, in part by National Natural Science Foundation of China under Grant (No. 61672342, 61671478, 61532012, 61822206, 61832013,  61960206002, 62041205), in part by Tencent AI Lab Rhino Bird Focused Research Program JR202034, in part by the Science and Technology Innovation Program of Shanghai (Grant 18XD1401800, 18510761200), in part by Shanghai Key Laboratory of Scalable Computing and Systems.

Zhanxing Zhu was supported by Beijing Nova Program (No. 202072) from Beijing Municipal Science \& Technology Commission, and National Natural Science Foundation  of  China  (No.61806009 and 61932001),  PKU-Baidu Funding 2019BD005.
\bibliography{amata}

\onecolumn
\appendix

\section{Appendix:Additional experiment details}
We use similar experiment settings as in \citep{zhang2019you} for MNIST and CIFAR10 experiment. 
\subsection{MNIST classification}
In this experiment, for PGD adversarial training, we set the adversarial constraint $\epsilon$ as 0.3, the step size as 0.01. For CW attack results in this experiment, we set the parameter eps to be 100 and run for 100 iterations in this reference implementation \url{https://github.com/xuanqing94/BayesianDefense/blob/master/attacker/cw.py}. For displaying the error-time curve, we use the smoothing function used in Tensorboard for all methods. The smoothing parameter is set as 0.09. For outer minimization, we use the stochastic gradient descent method and set the learning rate as 0.1, the momentum as 0.9, and the weight decay as 5e-4.

\subsection{Cifar10 classification}
In this experiment, for PGD adversarial training, we set the adversarial constraint $\epsilon$ as 8/255, the step size as 2/255. For CW attack results in this experiment, we set the parameter eps to be 0.5 and run for 100 iterations in this reference implementation \url{https://github.com/xuanqing94/BayesianDefense/blob/master/attacker/cw.py}. For displaying the error-time curve, we use the smoothing function used in Tensorboard for all methods. The smoothing parameter is set as 0.6. For outer minimization, we use the stochastic gradient descent method and set the learning rate as 5e-2, the momentum as 0.9, and the weight decay as 5e-4. We also use a piece-wise constant learning rate scheduler in PyTorch by setting the milestones at the 75-th and 90-th epoch with a factor of 0.1.

\section{Appendix:Additional experiment results}
\subsection{Evaluation against different settings of PGD on MNIST classification}
PGD-40 is chosen for evaluation, as it is a strong adversarial attack used in other papers \citep{zhang2019you}. We also try other settings, such as PGD-20 and PGD-60. Other settings are the same as our original settings in MNIST classification experiment. The results are shown in Figure~\ref{fig:exps_evalpgd}.
We can see that, compared with the case of PGD-40 for evaluation, both Amata and PGD-40 are able to achieve above the 96\% adversarial accuracy but at different speeds. In the case of PGD-20, Amata takes around 28\% training time of the PGD-40 to achieve the 96\% adversarial accuracy. In the case of PGD-60,  Amata takes around 60\% training time of the PGD-40 to achieve the 96\% adversarial accuracy.

\begin{figure}[!ht]

\begin{center}
\begin{tabular}{ll}
\includegraphics[width=0.5\columnwidth]{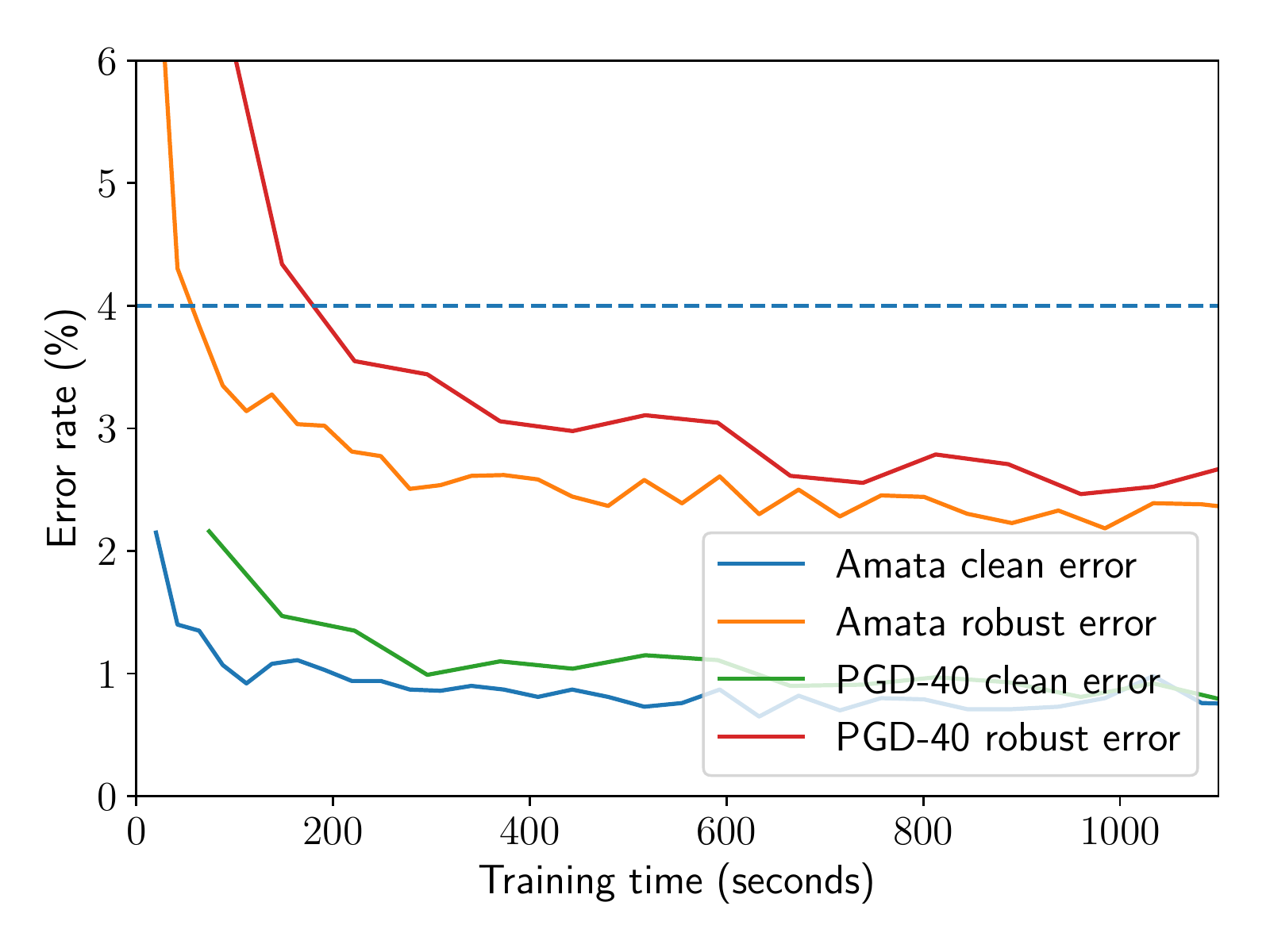} &
\includegraphics[width=0.5\columnwidth]{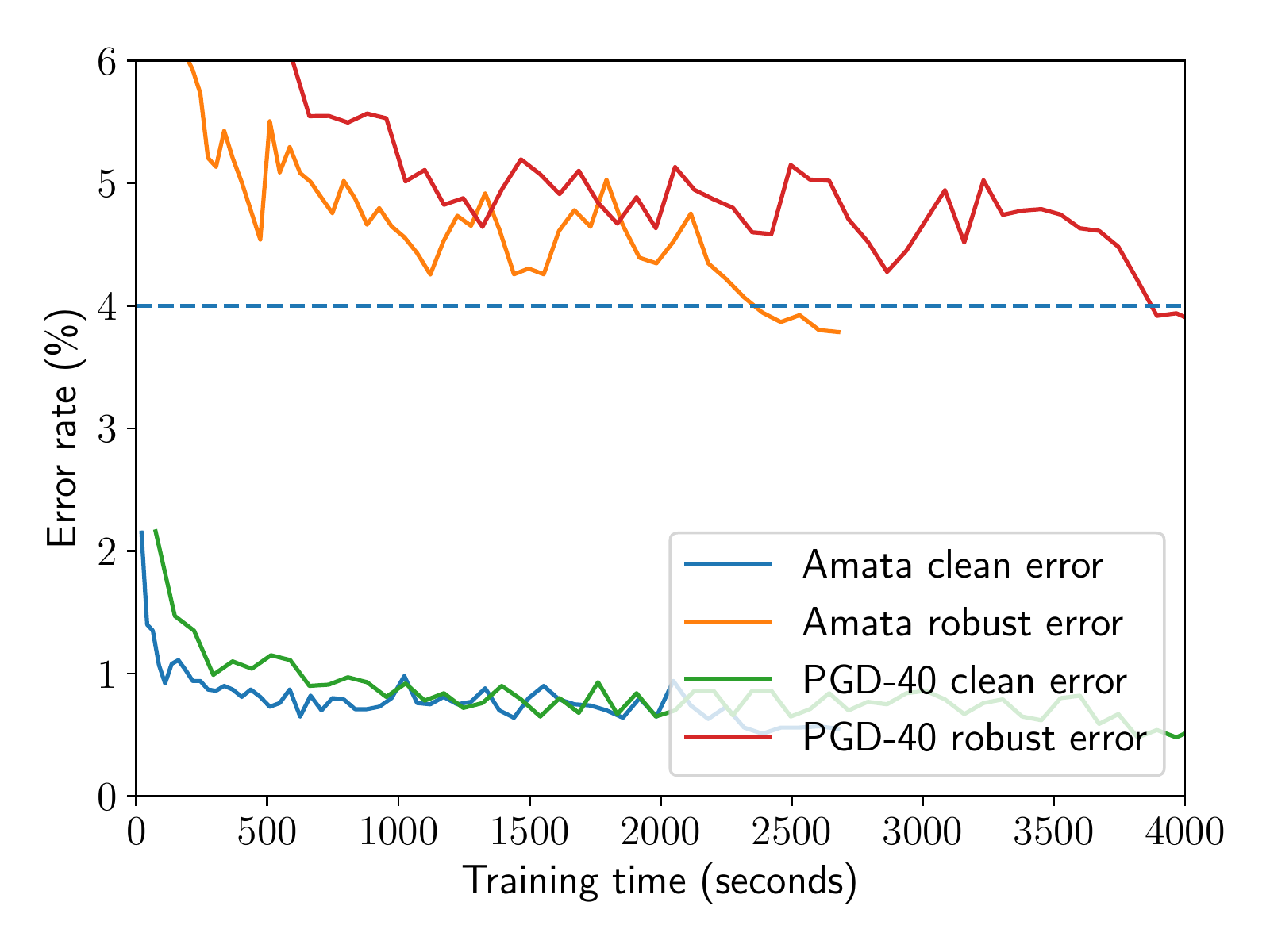}
\end{tabular}
\end{center}
\caption{\textbf{Left:} MNIST result. Training epoch against PGD-20 attack.\textbf{Right:} MNIST result. Training epoch against PGD-60 attack.}
\label{fig:exps_evalpgd}
\end{figure}

\subsection{Other decay scheme for controlling adversarial training}
Finding better decay scheme is an interesting future direction. We tried another exponential scheme:
\begin{equation}
K_{t} = K_{min} +  (K_{max}-K_{min})\cdot (1-e^{-\eta \cdot t})/(1-e^{-\eta \cdot T})
\end{equation}
where $K_{t}$ is the number of PGD adversarial training steps at $t$-th epoch, $K_{min}$ is the minimum number of PGD steps at the beginning, $K_{max}$ is the maximum number of PGD steps in the last epoch, $\eta$ is the hyper-parameter controlling the shape of the scheme. This design ensures that when $K_{0}= K_{min}$ and $K_{T} = K_{max}$, the exponential decay scheme is comparable with the linear scheme used in our paper. However, we tried different parameters of eta ranging from  0.1 to 5, but did not observe performance improvements. The results are shown in Figure~\ref{fig:MNIST_decayscheme}. The "Amata (Exp eta)" in the legend denotes the exponential decay scheme and the "Amata" in the legend denotes the original linear decay scheme. From Figure~\ref{fig:MNIST_decayscheme}, we can see that the new exponential decay scheme cannot outperform the linear decay scheme.

\begin{figure}[!h]
\begin{center}
\includegraphics[width=0.5\columnwidth]{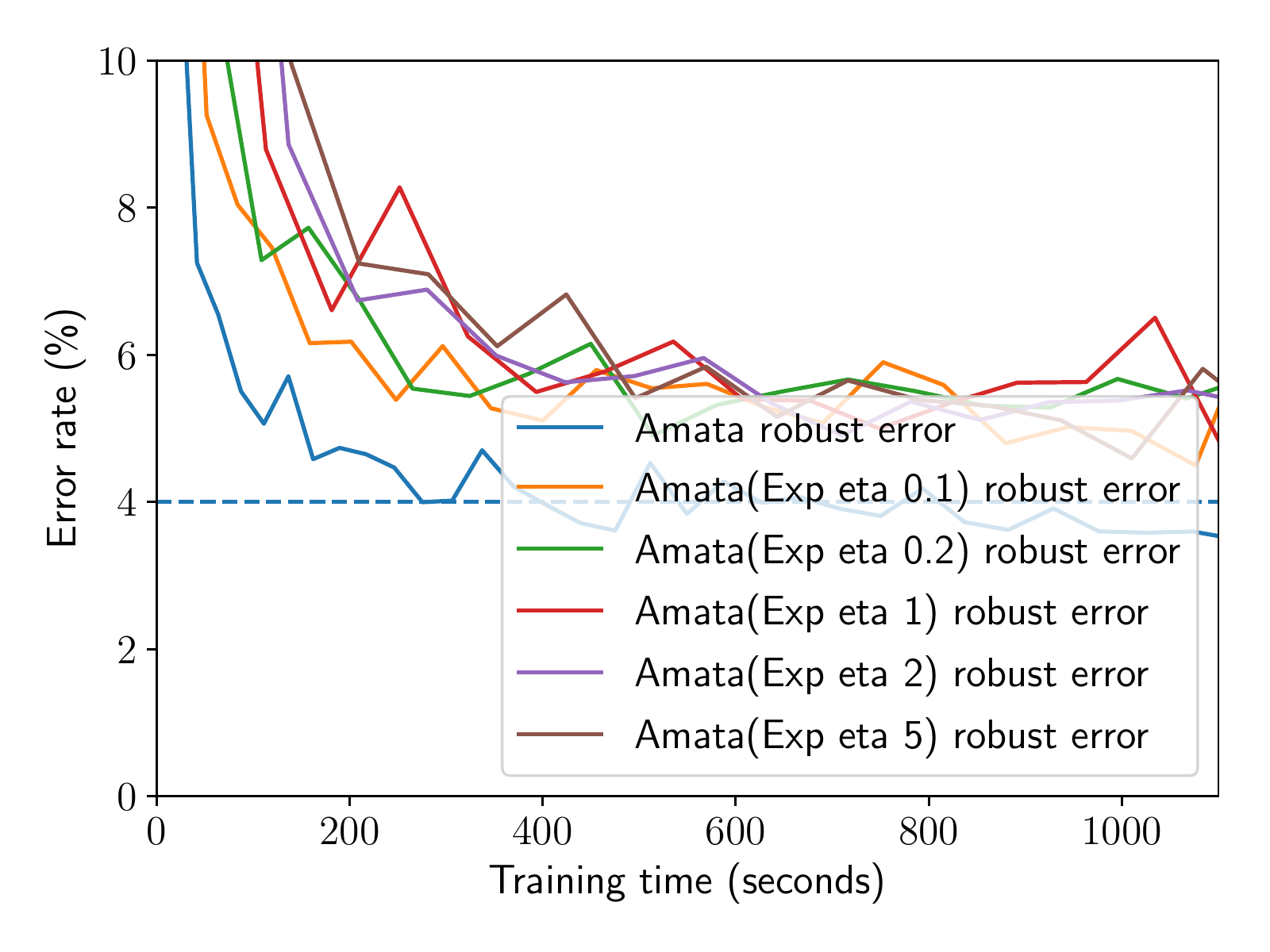}
\end{center}
\caption{MNIST result. Training epoch against PGD-40 attack. We use Amata with the setting $K_{min}=2$ and $K_{max}=5$.}
\label{fig:MNIST_decayscheme}
\end{figure}

\section{Visualizing and understanding adversarial training}
\label{sec:vis_amata}
\begin{figure}[!ht]
\begin{center}
\begin{tabular}{cc}
\includegraphics[width=0.4\columnwidth]{countour_epoch1.pdf} &
\includegraphics[width=0.4\columnwidth]{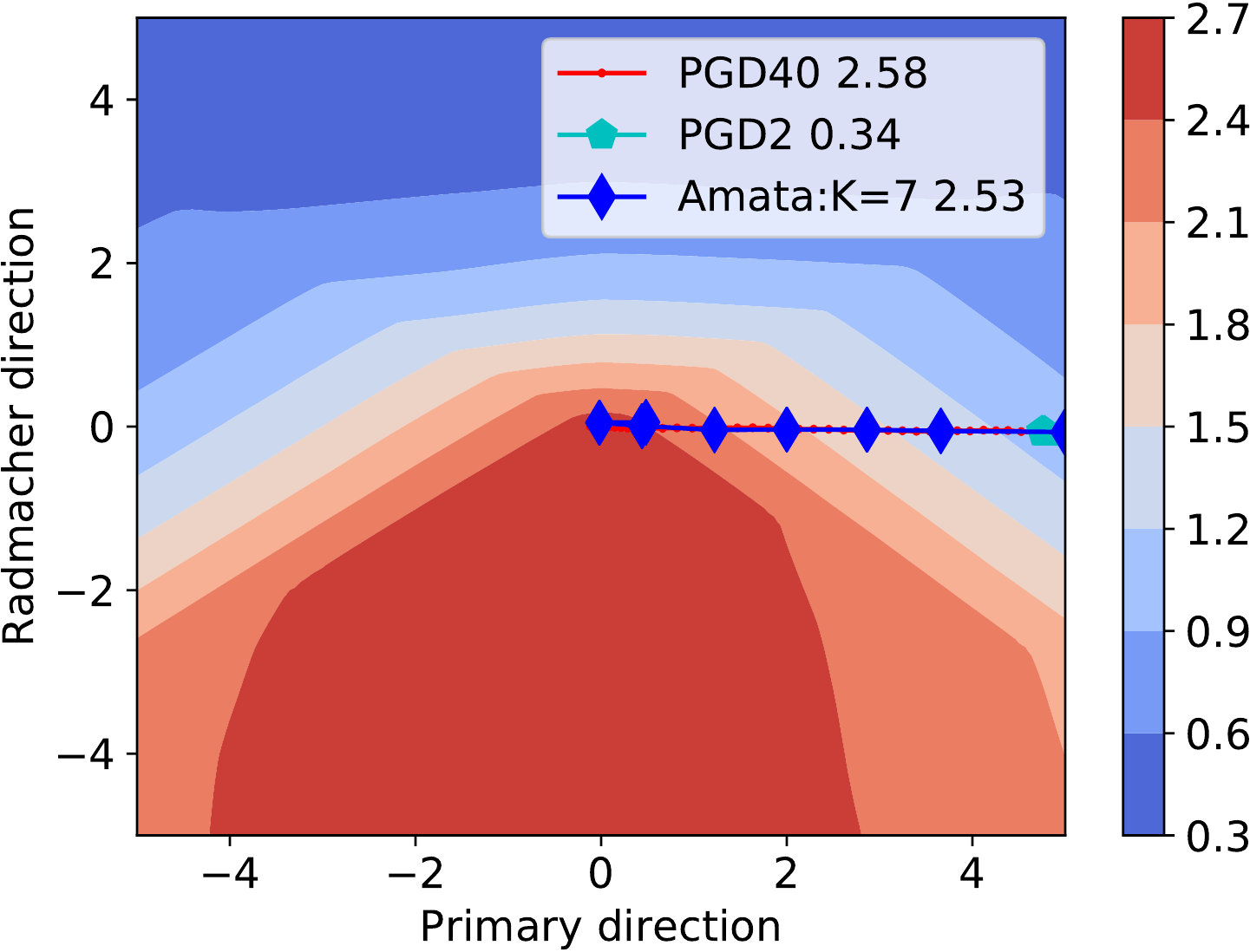} \\
(a) Epoch 1 &(b) Epoch 4  \\
\includegraphics[width=0.4\columnwidth]{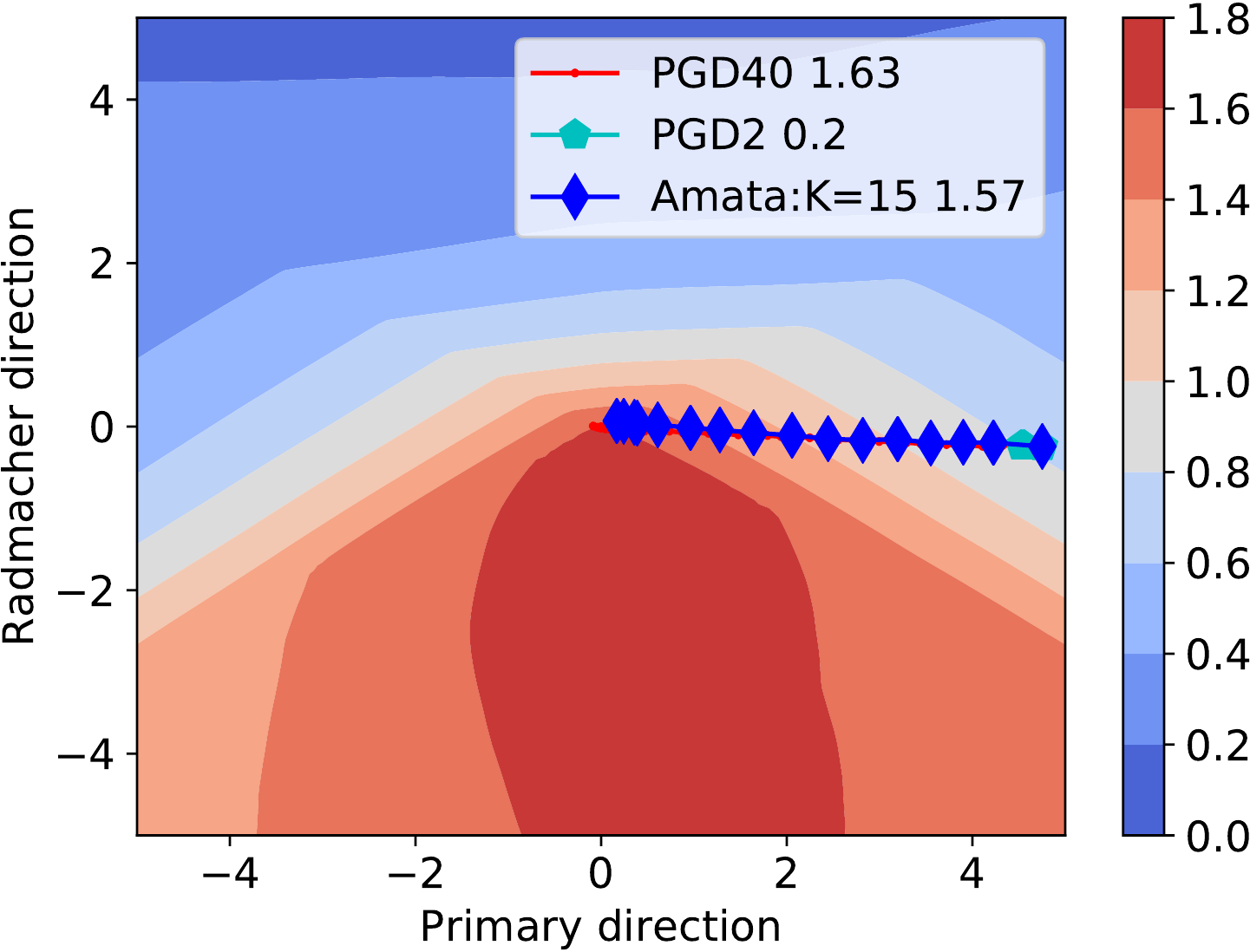} &
\includegraphics[width=0.4\columnwidth]{countour_epoch10.pdf} \\
(a) Epoch 8 &(b) Epoch 10
\end{tabular}
\end{center}
\caption{Visualization of inner maximization trajectories.}
\label{fig:vis_trajectory}
\end{figure}

We visualize and analyze the adversarial training process particularly on the inner maximization problem to give intuitions why Amata works. In the inner maximization steps, a sequence of adversarial examples are generated. They live in a high dimensional space that cannot be visualized directly. For visualization, we project the obtained sequence of adversarial examples on the X axis and Y axis while Z axis denotes the adversarial loss. For the X axis, we first obtain a set of vectors where each vector points from the normal example $\xB^{0}$ to the generated adversarial example in PGD. Then, similar to \citep{LiHaoNIPS2018}, we apply principal component analysis on the set of vectors and use the most principal component vector as the X axis. For the Y axis, similar to \citep{Ali2019}, we choose a random vector generated by the Rademacher distribution. We choose the MNIST adversarial training problem for illustration to reduce the computation burden. To generate the loss landscape of the adversarial training problem, we compute the adversarial loss value at each position of a $100\times 100$ grid. To plot the optimization trajectories of PGD and Amata in the low-dimensional space, we project the trajectories on the X axis and Y axis. For comparisons, we choose PGD-40, PGD-2 and Amata ($K_{\text{min}}=2$, $K_{\text{max}}=40$). The step size $\alpha$ is 0.01 for all methods. The results at different epochs are shown in Figure~\ref{fig:vis_trajectory}. We also indicate the adversarial loss obtained by different methods in the legend.The $K$ parameter in the legend for Amata is the number of steps for Amata. From Figure~\ref{fig:vis_trajectory}, we can observe that at the beginning of adversarial training (Epoch 1), the loss landscape is very smooth (difference of the maximum and the minimum is minor), making it easy to use large step size with fewer number of iterations. All methods can achieve similar adversarial loss. As the adversarial training proceeds, the loss landscape is becoming sharper (Epoch 4) and Amata can adaptively balance the number of steps and the step size to achieve a similar adversarial loss as obtained by PGD-40 at the cost of more iterations. During the process, we also observe that the maximum value of the adversarial loss landscape decreases gradually, indicating that better robustness is achieved. Note that in Epoch 8 and Epoch 10, PGD-2 can only achieve a adversarial perturbation resulting in low adversarial loss. From this visualization, we find Amata can adaptively achieve a trade-off between time cost and the degree of adversarial perturbation effectively.

\section{Appendix:Proof of convergence}
\label{sec:proof_converge}
We provide a convergence analysis of our proposed Amata algorithm for solving the min-max problem. The proof of convergence in this paper largely follows \citep{wang19i, sinha2018certifiable}. First, we will introduce some notations for clarity. We denote $\xB^{*}_{i}(\thetaB)=\argmax_{\xB_{i} \in \Xcal^{i}} \ell(\thetaB, \xB_{i})$ where $\ell(\thetaB, \xB_{i})$ is a short hand notation for the classification loss function  $\ell(h_{\thetaB}(\xB_{i}), y_{i})$, $\Xcal^{i}=\{\xB|\norm{\xB-\xB^{0}_{i}} \leq \epsilon \}$, and $\bar{f}_{i}(\thetaB)=\max_{\xB_{i} \in \Xcal^{i}} f(\thetaB, \xB_{i})$, then $\tilde{\xB_{i}}(\thetaB)$ is a $\delta$-approximate solution by our algorithm to $\xB^{*}_{i}(\thetaB)$, if it satisfies that:

\begin{equation}
    \label{eq:max_precision}
    \norm{\tilde{\xB_{i}}(\thetaB) - \xB^{*}_{i}(\thetaB)} \leq \delta
\end{equation}

In addition, denote the objective function in Equation~\ref{minmax_obj} by $L(\thetaB)$ , and its gradient by $\down L(\thetaB)=\frac{1}{n}\sum_{i=1}^{n}\down \ell_{i}(\thetaB)$. Let $g(\thetaB)=\frac{1}{|\Bcal|} \sum_{i \in \Bcal } \down_{\thetaB} \bar{f}(\thetaB)$ be the stochastic gradient of $L(\thetaB)$, where $\Bcal$ is the mini-batch. Then, we have $\Ebb[g(\thetaB)]=\down L(\thetaB)$. Let $\down_{\thetaB}\ell(\thetaB, \tilde{\xB}(\thetaB))$ be the gradient of $\ell(\thetaB, \tilde{\xB}(\thetaB))$ with respect to $\thetaB$, and $\tilde{g}(\thetaB) = \frac{1}{|\Bcal|} \sum_{i \in \Bcal} \down_{\thetaB} \ell(\thetaB, \tilde{\xB}_{i}(\thetaB)$ be the approximate stochastic gradient of $L(\thetaB)$. Before we prove the convergence of the algorithm, we have following assumptions.
\begin{assumption}
\label{assum:continuous}
The function $\ell(\thetaB, \xB)$ satisfies the gradient Lipschitz conditions:
\begin{align}
    \text{sup}_{\xB} \norm{\down_{\thetaB} \ell(\thetaB, \xB) -\down_{\thetaB} \ell(\thetaB^{*}, \xB) }_{2}  &\leq L_{\thetaB\thetaB} \norm{\thetaB - \thetaB^{*}}_{2} \nonumber\\
    \text{sup}_{\thetaB} \norm{\down_{\thetaB} \ell(\thetaB, \xB) -\down_{\thetaB} \ell(\thetaB, \xB^{*}) }_{2}  &\leq L_{\thetaB\xB} \norm{\xB - \xB^{*}}_{2} \nonumber\\
    \text{sup}_{\xB} \norm{\down_{\xB} \ell(\thetaB, \xB) -\down_{\xB} \ell(\thetaB^{*}, \xB) }_{2}  &\leq L_{\xB\thetaB} \norm{\thetaB - \thetaB^{*}}_{2} \nonumber
\end{align}
\end{assumption}
where $L_{\thetaB\thetaB}$, $L_{\thetaB\xB}$, and $L_{\xB\thetaB}$ are positive constants. Assumption~\ref{assum:continuous} was used in \citep{sinha2018certifiable,wang19i}.

\begin{assumption}
\label{assum:strongconvex}
The function $\ell(\thetaB, \xB)$ is \textit{locally $\mu$-strongly concave} in $\Xcal=\{\xB\,:\, \norm{\xB-\xB^{0}_{i}}_{\infty} \leq \epsilon \}$ for all $i \in [n]$, \ie, for any $\xB_{1}\, , \, \xB_{2} \in \Xcal_{i}$:
\begin{equation}
    \ell(\thetaB, \xB_{1}) \leq \ell(\thetaB, \xB_{2}) + \langle \down_{\xB}\ell(\thetaB, \xB_{2}), \xB_{1}-\xB_{2} \rangle - \frac{\mu}{2} \norm{\xB_{1}-\xB_{2}}^2_{2} \nonumber
\end{equation}
\end{assumption}
where $\mu$ is a positive constant which measures the curvature of the loss function. This assumption was used for analyzing distributional robust optimization problems \citep{sinha2018certifiable}.

\begin{assumption}
\label{assum:stochastic_variance}
The variance of the stochastic gradient $g(\thetaB)$ is bounded by a constant $\sigma^2>0$:
\begin{equation}
   \Ebb[\norm{g(\thetaB)-\down L(\thetaB)}^{2}_{2}] \leq \sigma^2 \nonumber
\end{equation}
where $\down L(\thetaB)$ is the full gradient.
\end{assumption}
The Assumption~\ref{assum:stochastic_variance} is commonly used for analyzing stochastic gradient optimization algorithms.

\begin{theorem}
\label{th:converge}
Suppose Assumptions~\ref{assum:continuous},\ref{assum:strongconvex}, and \ref{assum:stochastic_variance} holds. Denote $\Delta=L(\thetaB^{0})-\min_{\thetaB}L(\thetaB)$. If the step size of outer minimization is $\eta_{t}=\min(1/\beta, \sqrt{\frac{\Delta}{T\beta\sigma^2}})$. Then, we have:
\begin{equation}
    \frac{1}{T} \sum_{t=0}^{T-1} \Ebb[\norm{\down L(\thetaB^{t})}^{2}_{2}] \leq 4\sigma \sqrt{\frac{\beta\Delta}{T}} + 5L_{\thetaB\xB}^{2}\delta^2 \nonumber
\end{equation}
where $\beta=L_{\thetaB\xB}L{\xB\thetaB}/\mu + L_{\thetaB\thetaB}$.
\end{theorem}
The proof of this theorem can be found in the Appendix. Theorem~\ref{th:converge} indicates that if the Amata finds solutions of the inner maximization problem closer enough to the maxima, Amata can converge at a sublinear rate.

\begin{lemma}
\label{lemma:Ls-smooth}
Under Assumption~\ref{assum:continuous} and Assumption~\ref{assum:strongconvex}, we have $L(\thetaB)$ is $\beta$-smooth where $\beta=L_{\thetaB\xB}L_{\xB\thetaB}/\mu + L_{\thetaB\thetaB}$, \ie, for any $\thetaB_{1}$ and $\thetaB_{2}$ it holds:
\begin{align}
    L(\thetaB_{1}) \leq L(\thetaB_{2}) + \langle \down L(\thetaB_{2}), \thetaB_{1}-\thetaB_{2} \rangle + \frac{\beta}{2} \norm{\thetaB_{1}-\thetaB_{2}}^{2}_{2} \nonumber
\end{align}
\end{lemma}
The proof of Lemma~\ref{lemma:Ls-smooth} can be found in \citep{wang19i}.

\begin{lemma}
\label{lemma:gradbound}
Under Assumption~\ref{assum:continuous} and Assumption~\ref{assum:strongconvex}, the norm of difference between the approximate stochastic gradient $g(\thetaB)$ and the stochastic gradient $\tilde{g}(\thetaB)$ is bounded, \ie, it holds that:
\begin{equation}
    \norm{\tilde{g}(\thetaB)-g(\thetaB)}_{2} \leq L_{\thetaB\xB} \delta \nonumber
\end{equation}
\end{lemma}
\begin{proof}
We have
\begin{align}
    \norm{\tilde{g}(\thetaB) - g(\thetaB)}_{2} &= \norm{\frac{1}{|\Bcal|} \sum_{i \in \Bcal} \left(\down_{\thetaB}f(\thetaB, \tilde{\xB}_{i}(\thetaB)) - \down \bar{f}_{i}(\thetaB) \right)}_{2} \nonumber \\
    &\leq \frac{1}{|\Bcal|} \sum_{i \in \Bcal} \norm{\down_{\thetaB}f(\thetaB, \tilde{\xB}_{i}(\thetaB)) - \down_{\thetaB}f(\thetaB, \xB^{*}_{i}(\thetaB))}_{2} \nonumber \\
    &\leq \frac{1}{|\Bcal|} \sum_{i \in \Bcal} L_{\thetaB\xB} \norm{\tilde{\xB}_{i}(\thetaB) - \xB^{*}_{i}(\thetaB)}_{2}
\end{align}
where the first inequality is from the triangle inequality, and the second inequality is from Assumption~\ref{assum:continuous}. Next, we insert Equation~\ref{eq:max_precision} into the above inequality then:
\begin{equation}
     \norm{\tilde{g}(\thetaB) - g(\thetaB)}_{2} \leq L_{\thetaB\xB} \delta
\end{equation}
which completes the proof.
\end{proof}

Now we can prove the Theorem~\ref{th:converge}:
\begin{proof}
%\noindent \textit{Proof of Theorem~\ref{th:converge}.}

From Lemma~\ref{lemma:Ls-smooth}, we have:
\begin{align}
    &L(\thetaB^{t+1}) \leq L(\thetaB^{t}) + \langle \down L(\thetaB^{t}), \thetaB^{t+1}-\thetaB^{t} \rangle + \frac{\beta}{2} \norm{\thetaB^{t+1}-\thetaB^{t}}^{2}_{2} \nonumber \\
    &=L(\thetaB^{t})+\eta_{t} \langle \down L(\thetaB^{t}), \down L(\thetaB^{t}) - \tilde{g}(\thetaB^{t}) \rangle - \eta_{t} \norm{\down L(\thetaB^{t})}^{2}_{2}  + \frac{\beta\eta_{t}^{2}}{2}\norm{\tilde{g}(\thetaB^{2})}^{2}_{2} \nonumber \\
    &=L(\thetaB^{t}) -\eta_{t}(1-\frac{\beta\eta_{t}}{2}) \norm{\down L(\thetaB^{t})}^{2}_{2} + \eta_{t}(1-\frac{\beta\eta_{t}}{2}) \cdot  \langle \down L(\thetaB^{t}), \down L(\thetaB^{t}) - \tilde{g}(\thetaB^{t}) \rangle + \frac{\beta\eta_{t}^{2}}{2} \norm{\tilde{g}(\thetaB^{t})-\down L(\thetaB^{t})}^{2}_{2} \nonumber \\
    &=L(\thetaB^{t}) -\eta_{t}(1-\frac{\beta\eta_{t}}{2}) \norm{\down L(\thetaB^{t})}^{2}_{2} + \eta_{t}(1-\frac{\beta\eta_{t}}{2}) \cdot \nonumber \\& \langle \down L(\thetaB^{t}), \down L(\thetaB^{t}) -g(\thetaB^{t}) \rangle + \eta_{t}(1-\frac{\beta\eta_{t}}{2}) \langle \down L(\thetaB^{t}), \down L(\thetaB^{t}) - g(\thetaB^{t}) \rangle + \frac{\beta\eta_{t}^{2}}{2} \norm{\tilde{g}(\thetaB^{t})-g(\thetaB^{t})+g(\thetaB^{t})-\down L(\thetaB^{t})}^{2}_{2} \nonumber \\
    &\leq L(\thetaB^{t}) -\frac{\eta_{t}}{2}(1-\frac{\beta\eta_{t}}{2}) \norm{\down L(\thetaB^{t})}^{2}_{2} + \eta_{t}(1-\frac{\beta\eta_{t}}{2}) \cdot \nonumber \\& \norm{\tilde{g}(\thetaB) -g(\thetaB^{t})}^{2}_{2} \nonumber + \beta \eta_{t}^{2}( \norm{\tilde{g}(\thetaB^{t})-g(\thetaB^{t})}^{2}_{2}+\lvert\lvert g(\thetaB^{t})  -\down L(\thetaB^{t})\rvert\rvert ^{2}_{2})  + \eta_{t}(1+\frac{\beta \eta_{t}}{2}) \langle \down L(\thetaB^{t}), \down L(\thetaB^{t}) - g(\thetaB^{t}) \rangle \nonumber
\end{align}
Note that $\Ebb[g(\thetaB^{t})]=\down L(\thetaB)$, taking expectation on both sides of the inequality conditioned on $\thetaB^{t}$. Then we use Assumption~\ref{assum:stochastic_variance} and Lemma~\ref{lemma:gradbound} and simplify the above inequality:
\begin{align}
    &\Ebb[L(\thetaB^{t+1}) - L(\thetaB^{t}) | \thetaB^{t}]  \leq -\frac{\eta_{t}}{2}(1-\frac{\beta\eta_{t}}{2})\norm{\down L(\thetaB^{t})}^{2}_{2} \nonumber \\
    &+\frac{\eta_{t}}{2}(1+\frac{3\beta\eta_{t}}{2}) L_{\thetaB\xB}^{2}\delta^2 + \beta\eta_{t}^{2}\sigma^{2} \nonumber
\end{align}
Taking the telescope sum of the above equation from $t=0$ to $t=T-1$, we have
\begin{align}
    &\sum_{t=0}^{T-1} \frac{\eta_{t}}{2}(1-\beta\frac{\eta_{t}}{2}) \Ebb[\norm{\down L(\thetaB^{t})}^{2}_{2}] \leq \Ebb[\beta(\thetaB^{0}-\thetaB^{T})] \nonumber \\& + \sum_{t=0}^{T-1}\frac{\eta_{t}}{2}(1+\frac{3\beta\eta_{t}}{2})L_{\thetaB\xB}^{2}\delta^{2} + \beta\eta_{t}^{2}\sigma^{2} \nonumber
\end{align}
We set $\eta_{t}=\min(1/\beta, \sqrt{\frac{\Delta}{T\beta\sigma^2}})$, we have
\begin{align}
    \frac{1}{T} \sum_{t=0}^{T-1} \Ebb[\norm{\down L(\thetaB^{t})}^{2}_{2}] \leq 4\sigma \sqrt{\frac{\beta\Delta}{T}} + 5L_{\thetaB\xB}^{2}\delta^2 \nonumber
\end{align}

Thus, we complete the proof.
\end{proof}

\section{Appendix:Calculating the Optimal Control for the Toy Example}

Here, we calculate the control solution of the toy problem in the main paper. The loss function is
\begin{align}
    \ell(\theta,x) = \frac{\theta ^2}{2}-\frac{(x-\theta )^2}{\theta ^2+1},
\end{align}
where $x$ plays the role of data and $\theta$ plays the role of
trainable parameters. We will assume that the data point $x=0$ so that the non-robust loss is
\begin{align}
    l(\theta, 0) = \frac{\theta ^2}{2}+\frac{1}{\theta ^2+1}-1,
\end{align}
which has two minima at $\theta = \pm \sqrt{\sqrt{2} - 1}$. However, the robust loss is
\begin{align}
    \tilde{\ell}(\theta) = \max_{x \in \Rcal} l(\theta, x)
    = \frac{1}{2} \theta^2,
\end{align}
which has a unique minimum at $\theta=0$.

\paragraph{Inner Loop Optimization}

We have $\nabla_x \ell = -\frac{2 (x-\theta )}{\theta ^2+1}$ and so when we use step size ${{\alpha_t}}$ and we are at current parameter value
$\theta_t$, we have the inner loop iterations

\begin{align}
    x^{k} = x^{k-1} + {{\alpha_t}}
    \left(
        -\frac{2 (x^{k-1}-\theta_t )}{\theta ^2_t+1},
    \right), \qquad k=0,\dots,{K_t}-1 \quad ({K_t} {{\alpha_t}} = \tau), \qquad x_0 = 0
\end{align}

Note that we have ignored the sign and clipping steps for simplicity.
This has the exact solution

\begin{align}
    x({K_t}, {{\alpha_t}}) = \theta_t - \theta_t
    \left(
    1-\frac{2 {{\alpha_t}}}{\theta ^2+1}
    \right)^{\tau /{{\alpha_t}}}
\end{align}

\paragraph{Outer Loop Optimization and Control Problem}

For the outer loop, we will take as in the main paper the continuum approximation and treat $t$ as a continuous variable. In this case, we have the gradient flow

\begin{align}
    \dot{\theta}_t = - \nabla_\theta \ell(x({K_t},{{\alpha_t}}),\theta_t)
\end{align}

The control parameters are $u_t = ({K_t}, {{\alpha_t}})$ satisfying
$\tau = {K_t} {{\alpha_t}}$. The running cost is the number of inner loop
steps, which is simply $R(u_t) = \gamma {K_t} = \gamma \tau / {{\alpha_t}}$.
In addition, we will assume that $\tau$ is large so that the inner
loop is required to be stable, which places the constraint
$1-\frac{2 {{\alpha_t}}}{\theta ^2+1} \leq 1$. This can be thought of as
having a running barrier loss which is $+\infty$ unless the previous
constraints are satisfied. Hence, we arrive at the optimal control
problem

\begin{align}
    \min_{\{ u_t \}}
    \Phi(\theta(T)) + \int_{0}^{T} R(u_t) dt
\end{align}

where $\Phi(\theta) = \theta^2/2$ and $R(u_t) = \gamma {K_t} = \gamma \tau / {{\alpha_t}}$ subject to

\begin{align}
    &\dot{\theta}_t =  - \nabla_\theta
    \ell(x({K_t},{{\alpha_t}}),\theta_t)
    =
    -\frac{2 \theta_t  \left(1-\frac{2 \alpha}{\theta ^2+1}\right)^{\tau / \alpha}}{\theta_t ^2+1}+\frac{2 \theta_t ^3 \left(1-\frac{2 \alpha}{\theta_t ^2+1}\right)^{\frac{2 \tau }{\alpha}}}{\left(\theta_t ^2+1\right)^2}+\theta_t, \qquad \theta_0 \in \Rcal. \\
    &{{\alpha_t}} \in [1, (1+\theta_t^2) / 2] \qquad {K_t} {{\alpha_t}} = \tau
\end{align}

\paragraph{PMP Solution and Optimal Schedule.}

The Hamiltonian is
\begin{align}
    H(\theta,p,u) =
    p \left(\frac{2 \theta  \left(1-\frac{2 \alpha }{\theta ^2+1}\right)^{\tau /\alpha }}{\theta ^2+1}-\frac{2 \theta ^3 \left(1-\frac{2 \alpha }{\theta ^2+1}\right)^{\frac{2 \tau }{\alpha }}}{\left(\theta ^2+1\right)^2}-\theta \right)-\frac{\gamma  \tau }{\alpha }
\end{align}

And so the PMP equations are
\begin{align}
    &\dot{\theta}_t^* = \nabla_p H(\theta_t^*, p_t^*, u_t^*), \qquad \theta^*_0 = \theta_0 \\
    &p_t^* = - \nabla_x  H(\theta_t^*, p_t^*, u_t^*), \qquad p^*_T = - \theta_T \\
    &u_t^* = \mathrm{arg\,max}_{u} H(\theta^*_t, p^*_t, u)
\end{align}
where the last maximization is taken with respect to the constraints $\alpha \in [0,(1+{\theta^{*}_t}^2)]$. By writing out the right hand sides we can observe that $p^*_t$ is always of the opposite sign as $\theta^*_t$ (which do not change sign as it involves in time), and for $\tau > 1+\theta_0^2$ (which we choose $\tau$ to satisfy) we observe that $H$ is monotone increasing in $\alpha$, and so the maximum is attained at
\begin{align}
    \alpha^*_t = \frac{1}{2}(1+{\theta_t^*}^2).
\end{align}

This allows us to solve the PMP equations and yields the optimal inner
loop learning rate schedule
\begin{align}
    {\alpha_t}^* = \frac{1}{2}
    \left(
        1+{\theta_0^2} e^{-2 t}
    \right)
\end{align}

Note that this is also exactly the solution for which the optimal
control criterion yields a value of 0. Alternatively, the optimal
schedule for the inner loop steps is
\begin{align}
    {K_t}^*= \tau/{\alpha_t}^* = \frac{2\tau}{\theta_0^2 e^{-2t}+1}
    = \frac{2 \tau }{\theta_0 ^2+1}
    +
    \left(
	    \frac{4 \theta_0 ^2 \tau }{\left(\theta_0 ^2+1\right)^2}
    \right)
    t + \mathcal{O}(t^2)
\end{align}
This is an increasing schedule, similar to what we adopt in Amata. Of
course, in general the optimal control solution for real problems may
not be of this form so this only serves as a motivation for the
heuristic developed.

\section{Appendix:Approximation of the optimal control criterion}
\label{sec:app_criterion}
% In this section, we will explain how we derive the optimal control criterion using the method of successive approximations \citep{Chernousko1982a}.

In this section we show how to derive the approximation~\eqref{eq:simplified_criterion} from~\eqref{eq:criterion}. The primary assumption we make is that $T_2 - T_1 = \eta \ll 1$, which allows one to use a local approximation for various functions to derive a simple-to-compute criterion. For this reason, the derivation here will be largely heuristic.

First, we assume that we consider a constant control $\uB_s \equiv \uB$ on the interval $s \in [T_1,T_2] \equiv [t, t+\eta]$.
Next, recall that the adversarial robustness which serves as our terminal loss function for the control problem is $\Phi(\thetaB) = \max_{\zB} \ell ( h_{\thetaB} [\Acal_{\thetaB,\zB }(x)], y)$. In practice, any control that sufficiently conducts the adversarial perturbation serves as a good approximation. This is because the DNNs are very prone to adversarial attacks before the convergence of adversarial training and the real-time adversarial loss is close to the worst-case adversarial loss. Thus, we can approximate $\Phi(\thetaB)$ by $\ell ( h_{\thetaB} [\Acal_{\thetaB,\zB }(x)], y)$ by some chosen $\zB$ that makes practical computations easy.

Now, assuming regularity we can expand the state equation~\eqref{eq:pmp_state} to get to leading order as $\eta\rightarrow 0$
\begin{align}\label{eq:state_exp}
    \thetaB^{\uB}_{s} = \thetaB_{t} + o(1),
    \qquad
    s \in [t,t+\eta].
\end{align}
On the other hand, the co-state equation~\eqref{eq:pmp_costate} gives
\begin{align}\label{eq:costate_exp}
    \pB^{\uB}_{s}
    = -\nabla_{\thetaB}
    \Phi(\thetaB_{T_2}) + o(1)
    =
    - \nabla_{\thetaB} \ell ( h_{\thetaB_{t}} [\Acal_{\thetaB_{t},\zB}(x)], y) + o(1),
    \qquad
    s \in [t,t+\eta].
\end{align}
But, by definition $F(\thetaB,\uB) = -\nabla_{\thetaB} \ell ( h_{\thetaB} [\Acal_{\thetaB,\uB}(x)], y)$.
Then, for any $s\in[t,t+\eta]$, we have from~\eqref{eq:state_exp} and~\eqref{eq:costate_exp} that
\begin{align}
    \begin{split}
        H(\thetaB^{\uB}_s, \pB^{\uB}_s, \vB)
        &=
        \langle
            \pB^{\uB}_s,
            F(\thetaB^{\uB}_s, \vB)
        \rangle
        - R(\vB) \\
        &=
        \langle
            \nabla_{\thetaB} \ell ( h_{\thetaB_{t}}
            [\Acal_{\thetaB_t,\zB}(x)], y),
            \nabla_{\thetaB} \ell ( h_{\thetaB_{t}}
            [\Acal_{\thetaB_t,\vB}(x)], y)
        \rangle
        - R(\vB) + o(1)
    \end{split}
\end{align}
Note that the above expressions seem to be arbitrarily chosen from a strict mathematical sense, but these choices enable fast practical computation. Any of such expansions, assuming enough regularity, are equivalent in the limit $\eta\rightarrow 0$.

Plugging the above into the expression for $C$ (Eq.~\ref{eq:criterion}), we have
\begin{align}
    \begin{split}
        C(\uB, t)
        \approx&
        \max_{\vB}
        \left\{
            \langle
                \nabla_{\thetaB} \ell ( h_{\thetaB_{t}}
                [\Acal_{\thetaB_t,\vB}(x)], y),
                \nabla_{\thetaB} \ell ( h_{\thetaB_{t}}
                [\Acal_{\thetaB_t,\zB}(x)], y)
            \rangle
            - R(\vB)
        \right\}
        \\
        &-
        \left(
            \langle
                \nabla_{\thetaB} \ell ( h_{\thetaB_{t}}
                [\Acal_{\thetaB_t,\uB}(x)], y),
                \nabla_{\thetaB} \ell ( h_{\thetaB_{t}}
                [\Acal_{\thetaB_t,\zB}(x)], y)
            \rangle
            - R(\uB)
        \right)
    \end{split}
\end{align}
As discussed in the beginning of this section, to ease computation, we may replace $\zB$ by $\vB$ in the first term and $\uB$ in the second term, assuming that they give sufficient adversarial perturbations to get
\begin{align}
    \begin{split}
        C(\uB, t)
        \approx &
        \max_{\vB}
        \left\{
            \|
                \nabla_{\thetaB} \ell ( h_{\thetaB_{t}}
                [\Acal_{\thetaB_t,\vB}(x)], y)
            \|^2
            - R(\vB)
        \right\}
        \\
        &-
        \left(
            \|
                \nabla_{\thetaB} \ell ( h_{\thetaB_{t}}
                [\Acal_{\thetaB_t,\uB}(x)], y)
            \|^2
            - R(\uB)
        \right).
    \end{split}
\end{align}
Upon substituting $\vB=(K,\alpha)$ and $\uB=(K_t,\alpha_t)$ gives the desired result.

\end{document}